\documentclass[final]{siamltex}

\usepackage{graphics}
\usepackage{amsmath,amssymb,amsfonts}
\usepackage{mathrsfs}
\usepackage{wasysym}
\usepackage{algorithm}
\usepackage{algorithmic}
\usepackage{enumitem}
\usepackage{bm}

\usepackage{multirow,booktabs,comment}
\usepackage{epstopdf}

\usepackage{todonotes}
\usepackage[normalem]{ulem}

\usepackage{xcolor}
\usepackage{soul}

\usepackage[colorlinks=true,urlcolor=blue,citecolor=red,
anchorcolor=black,linkcolor=blue]{hyperref}

\usepackage{subcaption}

\usepackage{lineno}

\newtheorem{remark}[theorem]{Remark}

\newcommand{\bx}{\boldsymbol{x}}
\newcommand{\by}{\boldsymbol{y}}

\newcommand{\bY}{\boldsymbol{Y}}
\newcommand{\bZ}{\boldsymbol{Z}}
\newcommand{\bs}{\boldsymbol{s}}
\newcommand{\bt}{\boldsymbol{t}}
\newcommand{\ba}{\boldsymbol{a}}
\newcommand{\bb}{\boldsymbol{b}}
\newcommand{\bg}{\boldsymbol{g}}

\newcommand{\bz}{\boldsymbol{z}}
\newcommand{\bq}{\boldsymbol{q}}

\newcommand{\bv}{\boldsymbol{v}}
\newcommand{\bw}{\boldsymbol{w}}
\newcommand{\balpha}{\bm{\alpha}}
\newcommand{\bbeta}{\bm{\beta}}

\newcommand{\bgamma}{\bm{\gamma}}

\newcommand{\bxi}{\bm{\xi}}

\newcommand{\btheta}{\boldsymbol{\theta}}

\newcommand{\blambda}{\boldsymbol{\lambda}}

\newcommand{\mbI}{\mathbf{I}}

\newcommand{\mbbR}{\mathbb{R}}

\newcommand{\mbW}{\mathbf{W}}

\DeclareMathOperator*{\argmax}{arg\,max}

\usepackage{lineno}
 
\begin{document}
	
\title{Augmented KRnet for density estimation and approximation}
\author{Xiaoliang Wan\thanks{Department of Mathematics,
Center for Computation and Technology, Louisiana State
University, Baton Rouge 70803 (\tt xlwan@math.lsu.edu)}
\and
Kejun Tang\thanks{Peng Cheng Laboratory, Shenzhen, China ({\tt tangkj@pcl.ac.cn}).
}
}

\maketitle

\begin{abstract}
In this work, we have proposed augmented KRnets including both discrete and continuous models. One difficulty in flow-based generative modeling is to maintain the invertibility of the transport map, which is often a trade-off between effectiveness and robustness. The exact invertibility has been achieved in the real NVP using a specific pattern to exchange information between two separated groups of dimensions. KRnet has been developed to enhance the information exchange among data dimensions by incorporating the Knothe-Rosenblatt rearrangement into the structure of the transport map. Due to the maintenance of the exact invertibility, a full nonlinear update of all data dimensions needs three iterations in KRnet. To alleviate this issue, we will add augmented dimensions that act as a channel for the data dimensions to exchange information. In the augmented KRnet, a fully nonlinear update is achieved in two iterations. We also show that the augmented KRnet can be reformulated as the discretization of a neural ODE, where the exact invertibility is kept such that the adjoint method can be formulated with respect to the discretized ODE to obtain the exact gradient.  Numerical experiments have been implemented to demonstrate the effectiveness of our models. 
\end{abstract}

\section{Introduction}
Density estimation and approximation play an important role in many fields such as  variational Bayes, uncertainty quantification, unsupervised learning, etc. However, classical approaches or models such as kernel density estimator or the mixture of Gaussians are usually limited to low-dimensional cases due to the curse of dimensionality \cite{Scott2015}. Recently deep generative modeling has received a lot of attention in deep learning, which is closely related to density estimation. The main motivation of deep generative modeling is to deal with the distribution of data that have a very large number of dimensions, e.g., high-resolution images. So the deep generative modeling needs to balance modeling capability and efficiency. For example, generative adversarial networks (GANs) \cite{Goodfellow_2014,Arjovsky_2017} are able to learn a mapping from a latent space to the data space  without an explicit definition of likelihood. Due to such a flexibility, GANs have been successfully applied to many realistic applications; however, the lack of likelihood means that it is not suitable for density approximation, where a probability density function (PDF) is needed. Likelihood-based deep generative models have also been developed including the autoregressive models \cite{Graves_2013,Oord_2016a,Oord_2016b,Papamakarios_2018}, variational autoencoders (VAE) \cite{Kingma_2014,Kingma_2016}, and flow-based generative models \cite{Dinh_2014,Rezende_2015,Dinh_2016,Dhariwal_2018,Zhang_2018,Berg_2019}. The combination of different modeling strategies has also been actively explored. For instance, the flow-based model was coupled with GAN in \cite{Grover_2018} to obtain a likelihood;  The VAE, flow-based model and GAN were coupled in \cite{Zhu_2019} for more flexibility and efficiency; The flow-based model has been formulated as a discretized neural ordinary differential equations (ODE) \cite{Chen_2019,Dupont_2019}, where the velocity field of the ODE is modeled as a neural network. 

We pay particular attention to flow-based generative models. The underlying idea of flow-based generative models is to construct a transport map from the data distribution to a prior distribution, e.g., the standard Gaussian. There are two ways to define such a transport map: continuous and discrete models. The continuous models refer to the dynamical evolution given by a neural ODE, which transforms the distribution of the initial data to another distribution within a certain amount of time. In discrete models, 
the transport maps are explicitly constructed by stacking a sequence of simple bijections modeled by shallow neural networks. Both continuous and discrete models need to maintain the invertibility of the transport map. The transport map and its inverse determine two important things. One mapping direction yields the PDF model of the data distribution, which can be written as a product of the PDF of the prior distribution and the determinant of the Jacobian matrix; and the other mapping direction yields sample generation, which maps the samples generated by the prior distribution to samples that are consistent with the data distribution. Simply speaking, flow-based generative models provide an PDF model, which can be easily sampled. This is similar to classical probabilistic model such as the mixture of Gaussians. However, deep generative models are usually much more complex and capable. 

One interesting question is whether the flow-based generative model can serve as a generic PDF model for both density approximation and sample generation for problems in scientific computing. Note that density approximation and sample generation are usually addressed separately. To approximate a high-dimensional PDF, such as the posterior distribution in variational Bayes, the commonly used model is a Gaussian with diagonal or banded covariance matrix, which is often too simple although the statistics can be easily dealt with \cite{Blei_2018}. Given unlimited computational resources, sampling approaches such as Markov Chain Monte Carlo (MCMC) may eventually yield true samples for an arbitrary PDF under the assumption that the PDF is explicitly known up to a constant. It is challenging to compute the statistics of a high-dimensional random variable whose density is defined by a PDF equation, where both PDF approximation and sample generation may be expected simultaneously. We expect that the flow-based generative model can be capable enough to balance these two issues, e.g., we have applied the real NVP \cite{Dinh_2016} to importance sampling for efficient probability estimation for a PDE subject to uncertainty \cite{Wan_JCP20}, and KRnet \cite{Wan_KRnet} to approximate high-dimensional Fokker-Planck equations \cite{Wan_KRnet_FPE}. 

We have developed KRnet in \cite{Wan_KRnet} as a generalization of the real NVP \cite{Dinh_2016} by incorporating the triangular structure of the Knothe-Rosenblatt rearrangement into the definition of the transport map. The real NVP separates all data dimensions into two groups. When updating the current data, one group of dimensions can receive nonlinear information of the other group but only linear information of itself, which is a compromise to maintain the exact invertibility of the transport map. Using such a method, a fully nonlinear update needs three iterations. The main idea of KRnet was to enhance the exchange of information among data dimensions through a more flexible partition of data dimensions, which, however, does not change the mechanism of information exchange in each iteration. So KRnet cannot deal with one-dimensional data since two groups of data dimensions are needed. To alleviate this issue, we introduce augmented dimensions in this paper, which serve as a channel for the data dimensions to send and receive nonlinear information. The augmented KRnet achieves a fully nonlinear update in two iterations and is able to deal with one-dimensional data. We then reformulate the augmented KRnet such that it can be regarded as a discretization of an ODE, where the exact invertibility is kept in the discretization. The advantage of a discretization with exact invertibility is that the adjoint method can be formulated in terms of the discrete model instead of the ODE such that the computation of the gradient is exact. The drawback is that the accuracy of such a discretization is only of first order. 

The manuscript is organized as follows. In next section we briefly overview flow-based generative models and the KRnet. In section \ref{sec:aug_KRnet}, we define augmented KRnet including both discrete and continuous models. Numerical experiments are implemented to demonstrate the effectiveness of the proposed strategies in section \ref{sec:num}, followed by a summary section.  

\section{KRnet}
KRnet is a discrete flow-based generative model. Generally speaking, the key component of a flow-based generative model is an invertible mapping $f(\cdot):\mathbb{R}^n\rightarrow\mathbb{R}^n$:
\begin{align*}
	\bz&=f(\by)=f_{[m]}\circ f_{[m-1]}\circ\ldots\circ f_{[i]}\circ\ldots f_{[2]}\circ f_{[1]}(\by),\\
	\by&=f^{-1}(\bz)=f^{-1}_{[1]}\circ f^{-1}_{[2]}\circ\ldots\circ f^{-1}_{[i]}\circ\ldots f^{-1}_{[m-1]}\circ f^{-1}_{[m]}(\bz),
\end{align*} 
which can be regarded as a composite mapping that consists of a sequence of intermediate bijections $f_{[i]}(\cdot):\mathbb{R}^n\rightarrow\mathbb{R}^n$. Let $p_{\bZ}$ and $p_{\bY}$ be the probability density functions (PDF) of the random variables $\bZ$ and $\bY$ respectively. The transformation $\bZ=f(\bY)$ yields the following relation
\begin{equation}\label{eqn:pdf_model}
	p_{\bY}(\by)=p_{\bZ}(f(\by))\left|\det\nabla_{\by} f(\by)\right|.
\end{equation}
In other words, if we associate $\bZ$ with a certain prior distribution, e.g., the standard Gaussian distribution, the mapping $f(\cdot)$ induces explicitly a PDF model $p_{\bY}$, which can be used for either density estimation or approximation. Furthermore, the distribution $p_{\bY}$ can be easily sampled as $\by^{(i)}=f^{-1}(\bz^{(i)})$ thanks to the invertibility of $f(\cdot)$, where $\bz^{(i)}$ is sampled from the prior distribution $p_{\bZ}$. Flow-based generative models share two main features: a large number of intermediate mappings $f_{[i]}(\cdot)$ and invertibility. Since $f(\cdot)$ intends to map the prior to an arbitrary distribution, a large number of intermediate mappings implies that the complexity of $f_{[i]}(\cdot)$ can be reduced. There are two ways to obtain $f_{[i]}(\cdot)$. One is explicit construction, e.g., NICE \cite{Dinh_2014}, real NVP \cite{Dinh_2016}, planar flow \cite{Rezende_2015}, inverse autoregressive flow \cite{Kingma_2016}, Sylvester flow \cite{Berg_2019}, and KRnet \cite{Wan_KRnet}; the other one is through the discretization of a continuous model, e.g.,  neural ODE \cite{Chen_2019} and its variants subject to either augmentation \cite{Dupont_2019} or regularization \cite{Yang_2019,Finlay_2020}. Depending on the way to obtain $f_{[i]}(\cdot)$, we may classify the flow-based generative models as discrete or  continuous models. Invertibility plays an important role because density estimation and sample generation use opposite directions of the same mapping. The invertibility of discrete models is usually exact and maintained locally by each $f_{[i]}(\cdot)$ while the invertibility of continuous models may only be kept at the continuous level and is not exact locally after the continuous model is discretized. For instance, for neural ODEs the two directions of the mapping $f(\cdot)$ correspond to forward and backward integration of the ODE. It is well known that although an ODE is theoretically invertible there does not exist a numerical scheme which is exactly invertible.   

In terms of the optimal transport theory, the mapping $f(\cdot)$ corresponds to a transport map. Let $\mu_{\bY}$ and $\mu_{\bZ}$ indicate the probability measures of $\bY$ and $\bZ$, respectively. The mapping ${T}:\bZ\rightarrow\bY$ is called a transport map such that ${T}_\#\mu_{\bZ}=\mu_{\bY}$, where ${T}_{\#}\mu_{\bZ}$ is the push-forward of the law $\mu_{\bZ}$ of $\bZ$ such that $\mu_{\bY}(B)=\mu_{\bZ}({T}^{-1}(B))$ for every Borel set $B$ \cite{Filippo2010}. It is seen that $T$ can be defined as $T(\bz)=f^{-1}(\bz)$.  The Knothe-Rosenblatt (K-R) rearrangement says that a transport map may have a lower-triangular structure such that \cite{Spantini_2017}
\begin{equation}\label{eqn:KR}
\bz={T}^{-1}(\by)=f(\by)=\left[
\begin{array}{l}
f_1(y_1)\\
f_2(y_1,y_2)\\
\vdots\\
f_n(y_1,y_2,\ldots,y_n)
\end{array}
\right].
\end{equation}  
Such a mapping can be regarded as a limit of a sequence of optimal transport maps when the quadratic cost degenerates \cite{Carlier_2010}. We have defined a flow-based generative model called KRnet in \cite{Wan_KRnet,Wan_KRnet_FPE} which generalizes the real NVP \cite{Dinh_2016} by adapting the triangular structure of the K-R rearrangement into the model. For more flexibility, we consider a partition $\by=(\by_1,\ldots,\by_K)$ in KRnet, where $\by_i=(y_{i,1},\ldots,y_{i,m})$, $1\leq K\leq n$, $1\leq m\leq n$, and $\sum_{i=1}^K\textrm{dim}(\by_i)=n$. We employ a block-version of the K-R rearrangement
\begin{equation}\label{eqn:KR-block}
\bz=f(\by)=\left[
\begin{array}{l}
f_1(\by_1)\\
f_2(\by_1,\by_2)\\
\vdots\\
f_K(\by_1,\ldots,\by_K)
\end{array}
\right].
\end{equation}  
To integrate the K-R rearrangement into the data flow of $f(\cdot)$, we need to associate the mappings $f_i(\cdot)$, $i=1,\ldots,K$, with an order. We let the data flow from $f_K$ to $f_1$:
\[
\by\xrightarrow[]{f_K}\by_{t_1}\xrightarrow[]{f_{K-1}}\by_{t_2}\xrightarrow[]{f_{K-2}}\ldots\by_{t_{K-1}}\xrightarrow[]{f_1}\by_{t_K}=\bz.
\]
At step $t_i$, a certain group of dimensions will be deactivated. Thus KRnet has a lower triangular overall structure in the sense that the number of effective dimensions decreases similarly to the transition from $f_{K}(\cdot)$ to $f_1(\cdot)$ in equation \eqref{eqn:KR-block}. 

\subsection{An overview of the layers in KRnet}\label{sec:layers_krnet}
We now briefly overview some main building blocks $f_{[i]}(\cdot)$ used in KRnet. More details can be found in \cite{Wan_KRnet,Wan_KRnet_FPE}. We let $\by_{[i]}$ indicate an intermediate state of data after the mapping $f_{[i-1]}(\cdot)$, i.e., $\by_{[i]}=f_{[i]}(\by_{[i-1]})$ with $\by_{[0]}=\by$.

\noindent\emph{1. Squeezing layer} deactivates a certain number of components using a mask 
\[
{\bq} = (\underbrace{1,\ldots,1}_{k},\underbrace{0,\ldots,0}_{n-k}).
\]
The first $k$ components given by $\bq\odot\by_{[i]}$ will keep being updated and the rest $(n-k)$ components given by $(1-\bq)\odot\by_{[i]}$ will be deactivated from then on. Here  $\odot$ indicates the Hadamard product or component-wise product. 

\noindent\emph{2. Rotation layer} provides a simple and trainable strategy to determine the dimensions that will be deactivated first. The rotation layer defines a rotation of the coordinate system through an orthogonal matrix for the current active dimensions:
	\[
	{\by}_{[i+1]}=\hat{\mbW}\by_{[i]}=\left[
	\begin{array}{cc}
	\mbW&0\\
	0&\mbI
	\end{array}
	\right]\by_{[i]}=\left[
	\begin{array}{cc}
	\mathbf{L}&0\\
	0&\mbI
	\end{array}
	\right]
	\left[
	\begin{array}{cc}
	\mathbf{U}&0\\
	0&\mbI
	\end{array}
	\right]\by_{[i]},
	\]
	where $\mbW\in\mbbR^{k\times k}$ with $k$ being the number of 1's in ${\bq}$, and $\mbI\in\mbbR^{(n-k)\times(n-k)}$ is an identity matrix, and $\mbW=\mathbf{LU}$ is the LU decomposition of $\mbW$. The entries in the lower triangular part of $\mathbf{L}$ and the upper triangular part of $\mathbf{U}$ will be treated as trainable parameters of the model except for the diagonal entries of $\mathbf{L}$ which are equal to 1. Intuitively we expect the rotation may put the most important dimensions at the beginning, which need further modifications. We need to clarify one thing. Although the purpose of this layer can be understood through a rotation of the coordinate system, we simply train $\mathbf{L}$ and $\mathbf{U}$ in practice without enforcing the unity of $\hat{\mbW}$. 

\noindent\emph{3. Scale and bias layer} provides a simplification of batch normalization which is defined as  \cite{Szegedy_2015,Dhariwal_2018}
	\begin{equation}\label{eqn:actnorm}
	{\by}_{[i+1]} = \ba\odot\by_{[i]}+\bb, 
	\end{equation}
	where $\ba$ and $\bb$ are trainable, and initialized by the mean and standard deviation of data. After the initialization, $\ba$ and $\bb$ will be treated as regular trainable parameters that are independent of the data. The scale and bias layer helps to improve the conditioning of deep nets.

\noindent\emph{4. Affine coupling layer} is the most important layer for evolving the data. Consider a partition 
	$\by_{[i]}=(\by_{[i],1},\by_{[i],2})$ with $\by_{[i],1}\in \mathbb{R}^m$ and $\by_{[i],2}\in\mathbb{R}^{n-m}$. The affine coupling layer is  defined as \cite{Wan_KRnet,Dinh_2016}
	\begin{equation}\label{eqn:new_affine}
	\left\{
	\begin{array}{ll}
      &\bz_1=\by_{[i],1},\\
      &\bz_2=\by_{[i],2}\odot(1+\alpha\tanh(\bs(\by_{[i],1}))+e^{\bbeta}\odot\tanh(\bt(\by_{[i],1})),
    \end{array}\right.
	\end{equation}
	where $\bs,\bt\in\mathbb{R}^{n-m}$ stand for scaling and translation functions depending only on $\by_{[i],1}$, $0<\alpha<1$ and $\bbeta\in\mathbb{R}^n$. Note that $\by_{[i],2}$ is updated linearly while the mappings $\bs(\by_{[i],1})$ and $\bt(\by_{[i],1})$ can be arbitrarily complicated, which are modeled as a neural network (NN),
	\begin{equation}\label{eqn:NN}
	(\bs, \bt) = \textsf{NN}(\by_{[i],1}).
	\end{equation}
	Then the Jacobi matrix is lower-triangular, and an inverse can be easily computed. The two parts of $\by_{[i]}$ will be updated alternatingly by a sequence of affine coupling layers, e.g., at the next affine coupling layer, the first partition will be modified while the second partition remains fixed.

\noindent\emph{5. Nonlinear invertible layer} defines a component-wise one-dimensional nonlinear mapping to map $\mathbb{R}$ to itself. We decompose $\mathbb{R}=(-\infty,-a)\cup[-a,a]\cup(a,\infty)$ for $0<a<\infty$, and define
	\begin{equation}\label{eqn:nonlinear_layer_def}
z=\hat{F}(y)=\left\{
\begin{array}{rl}
\beta (y-a)+a,&y\in(-\infty,-a)\\
\phi^{-1}\circ F\circ\phi(y),&y\in[-a,a]\\
\beta (y+a)-a,&y\in(a,\infty),
\end{array}
\right.
\end{equation}
	where $\phi:[-a,a]\rightarrow[0,1]$ is an affine mapping, $\beta>0$ is a scaling factor, and 
	\begin{equation}
	F(x)=\int_0^xp(x)dx,\quad \forall x\in[0,1].
	\end{equation}
Here $p(x)$ can be regarded a PDF and $F(x)$ a cumulative distribution function. In particular, $p(x)$ will be defined as a piecewise linear function  such that $F(x)$ is a quadratic function whose inverse can be computed explicitly.

\subsection{Main structure of KRnet}
The main structure of KRnet is illustrated in Figure \ref{fig:structure_diagram}. 
\begin{figure}	\center{
		\includegraphics[width=0.58\textwidth]{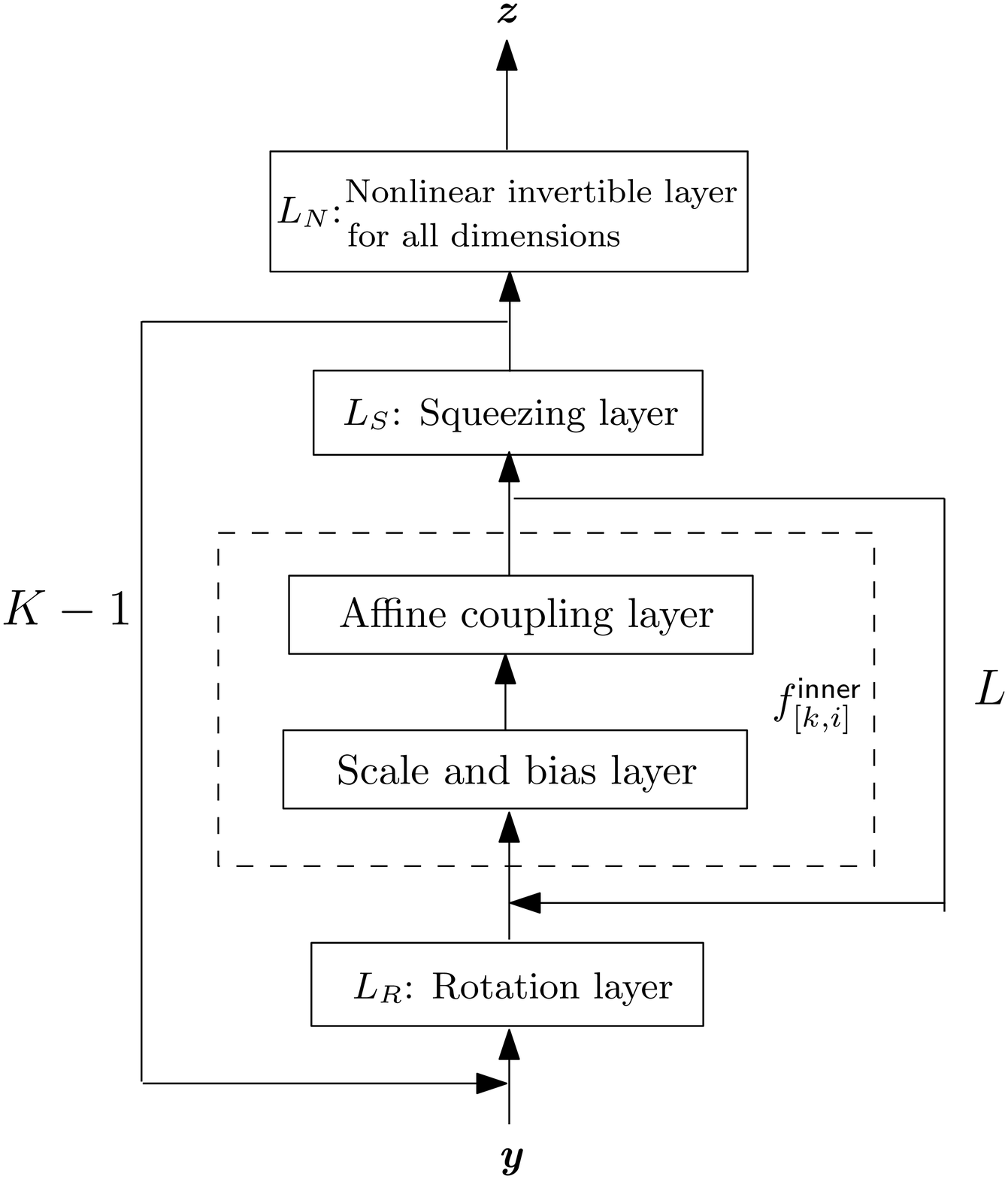}
	}
	\caption{Left: the flow chart of the block-triangular invertible mapping KRnet.}\label{fig:structure_diagram}
\end{figure}
KRnet is mainly defined by two loops: outer loop $f_{[k]}^{\mathsf{outer}}$ and inner loop $f_{[k,i]}^{\mathsf{inner}}$, $k=1,\ldots,K-1;\,i=1,\ldots,L$, where the outer loop has $K-1$ stages induced by the $K$ mappings $f_i$ in equation \eqref{eqn:KR-block}, and the inner loop has $L$ stages indicating the length of a chain that consists of general coupling layers.
\begin{itemize}
	\item Outer loop. The outer loop defines the main structure of KRnet that is consistent with the KR arrangement:
	\begin{equation}
	\bz=f(\by)=L_N\circ f^{\mathsf{outer}}_{[K-1]}\circ\ldots\circ f^{\mathsf{outer}}_{[1]}(\by).
	\end{equation}
	Let $\by_{[k]}=f^{\mathsf{outer}}_{[k]}(\by_{[k-1]})$ with $\by_{[0]}=\by$, and $i=1,\ldots,K-1$. Each $\by_{[k]}=(\by_{[k],1},\ldots,\by_{[k],K})$ has the same partition. The $i$th partition will remain unchanged after $K-i+1$ stages. For example, $\by_{[k],K}$ will be updated only when $k=1$ and $\by_{[k],K-i}$ will be deactivated when $k>i+1$. This way, the number of effective dimensions decreases as $k$ increases. Assuming that the prior distribution is chosen as a standard Gaussian, all dimensions of $\by_{[K-1]}$ are supposed to be independent of each other after the outer loop is completed. We then activate all the dimensions and apply the nonlinear invertible layer to $\by_{[K-1]}$ component-wisely before the final output. The nonlinear invertible layer generalizes the prior distribution through a component-wise nonlinear transformation.  
	\item Inner loop. The inner loop mainly consists of a sequence of general coupling layers $f^{\mathsf{inner}}_{[k,i]}$.  Each $f^{\mathsf{inner}}_{[k,i]}$ includes one scale and bias layer, and one affine coupling layer. $f^{\mathsf{outer}}_{[k]}$ can be represented as:
	\begin{equation}
	f^{\mathsf{outer}}_{[k]}=L_S\circ f^{\mathsf{inner}}_{[k,L]}\circ\ldots\circ f^{\mathsf{inner}}_{[k,1]}\circ L_R,
	\end{equation}
	where $L_R$ is a rotation layer, and $L_S$ is a squeezing layer. The affine coupling layers in $f_{[k,i]}^{\mathsf{inner}}$ are defined between $\by_{[k],K+1-k}$ and the other active parts $\by_{[k],i}$, $i=1,\ldots,K-k$. Since we need at least two affine coupling layers for a full update of all data dimensions, we usually assume that $L$ is even.
\end{itemize}

\section{Augmented KRnet}\label{sec:aug_KRnet}
In affine coupling layers \eqref{eqn:new_affine} we need to update a certain part of the data using a mapping of the other part. The main motivation of such a strategy is to maintain the exact invertibility. The main drawback of such a strategy is that the change of a certain component $y_i$ for each update is at most a linear function of $y_i$ (see equation \eqref{eqn:new_affine}). To alleviate such a limitation, we implement the affine coupling layer in a higher dimensional space such that the update of $y_i$ may be in terms of all the components of $\by$ through the augmented dimensions. 

\subsection{Introduce augmented dimensions}
Suppose that $\{\by^{(i)}\}_{i=1}^N$ consists of samples from $\bY\in\mathbb{R}^n$ subject to a PDF $p_{\bY}(\by)$. We augment $\bY$ by another vector $\bgamma\in\mathbb{R}^m$ such that $\bY_{\bgamma}=(\bgamma,\bY)$. Let $\bZ\in\mathbb{R}^n$ have the prior distribution $p_{\bZ}(\bz)$. The random variable $\bZ$ is augmented similarly by a vector $\bxi$ such that $\bZ_{\bxi}=(\bxi,\bZ)$. Instead of considering an invertible mapping between $\bY$ and $\bZ$, we construct an invertible mapping between $\bY_{\bgamma}\in\mathbb{R}^{n+m}$ and $\bZ_{\bxi}\in\mathbb{R}^{n+m}$ such that 
\begin{equation}
\bz_{\bxi}=f_{\mathsf{aug}}(\by_{\bgamma}).
\end{equation}
Assuming that $\bY$ and $\bgamma$ are independent. We have the PDF of $\bY_{\bgamma}$ as
\begin{equation}\label{eqn:pdf_aug}
p_{\bY}(\by)p_{\bgamma}(\bgamma)=p_{\bZ_{\bxi}}(f_{\mathsf{aug}}(\by_{\bgamma}))|\nabla_{\by_{\bgamma}}f_{\mathsf{aug}}(\by_{\bgamma})|,\quad\forall\by_{\bgamma}.
\end{equation}
We now look at how the information flows through the augmented dimensions. Applying the affine coupling layers to the partition given by the data dimensions and the augmented dimensions, we have two adjacent affine coupling layers as 
	\begin{align}
      \bgamma_{[i+1]}=&\bgamma_{[i]}\odot\bw_{[i]}(\by_{[i]})+\bb_{[i]}(\by_{[i]}),\\
      \by_{[i+1]}=&\by_{[i]}
	\end{align}
and
	\begin{align}
	      \bgamma_{[i+2]}=&\bgamma_{[i+1]},\\
	  \by_{[i+2]}=&\by_{[i+1]}\odot\bw_{[i+1]}(\bgamma_{[i+1]})+\bb_{[i+1]}(\bgamma_{[i+1]}),
	\end{align}
where we let
\begin{align*}
\bw_{[i]}(\by_{[i]})=&1+\alpha\tanh(\bs_{[i]}(\by_{[i]}),\\
\bb_{[i]}(\by_{[i]})=&e^{\bbeta_{[i]}}\odot\tanh(\bt_{[i]}(\by_{[i]})).
\end{align*}
We observe the following flow of information:
\[
\by_{[i]}\,\,\rightarrow\,\,\bgamma_{[i+1]}\,\,\rightarrow\,\,\by_{[i+2]}\,\,\rightarrow\,\,\bgamma_{[i+3]}\,\,\rightarrow\ldots
\]	
which implies that $\by_{[i+2]}$ may be affected by all the components of $\by_{[i]}$ although such a dependence is not explicit. 

Similarly, the two adjacent steps in a regular KRnet can be written as
	\begin{align*}
      \by_{[i+1],1}=&\by_{[i],1}\odot\bw_{[i]}(\by_{[i],2})+\bb_{[i]}(\by_{[i],2})),\\
      \by_{[i+1],2}=&\by_{[i],2}
	\end{align*}
and
	\begin{align*}
	      \by_{[i+2],1}=&\by_{[i+1],1},\\
	  \by_{[i+2],2}=&\by_{[i+1],2}\odot\bw_{[i+1]}(\by_{[i+1],1}))+\bb_{[i+1]}(\by_{[i+1],1})),
	\end{align*}
where $\by_{[i]}=(\by_{[i],1},\by_{[i],2})^{\mathsf{T}}$ has been partitioned to two parts. It is seen that after two steps, $\by_{[i+2]}$ will not depend on $\by_{[i]}$ in a fully nonlinear way, where $\by_{[i+2],1}$ depends on $\by_{[i],1}$ linearly and only $\by_{[i+2],2}$ depends on both $\by_{[i],1}$ and $\by_{[i],2}$ nonlinearly. 

Although more nonlinear dependence of $\by_{[i+2]}$ on $\by_{[i]}$ has been introduced through the augmented dimensions, we need to deal with $(m+n)$-dimensional mapping, implying a higher requirement on the complexity of the model due to the curse of dimensionality. Both issues are related to the choice of $m$. We note that the dependence of $\by_{[i+2]}$ on $\by_{[i]}$ prefers a larger $m$ while a smaller $m$ is preferred from the viewpoint of model complexity. The KRnet provides an effective way to balance these two issues. 

Suppose that KRnet uses a uniform partition of $\by\in\mathbb{R}^n$ as $\by=(\by_1^\mathsf{T},\ldots,\by_K^\mathsf{T})^\mathsf{T}$ with $K=n/m$. We then consider an augmented vector $\by_{\bgamma}=(\bgamma^{\mathsf{T}},\by^\mathsf{T})^\mathsf{T}$ with $\bgamma\in\mathbb{R}^m$, i.e., we let $\by_i$ and $\bgamma$ have the same number of dimensions. 
The overall structure of the augmented KRnet is illustrated in figure \ref{fig:augmented_structure_KR}, which is similar to the regular KRnet. The main difference is that the augmented part will never be deactivated since it is used as a buffer zone for communicating information. Due to the refined partition of data in KRnet, only a small  number of augmented dimensions is needed. For example, if we deactivate dimensions one by one in KRnet, we only need one extra dimension, i.e., $m=1$. 
\begin{figure}	\center{
		\includegraphics[width=0.9\textwidth]{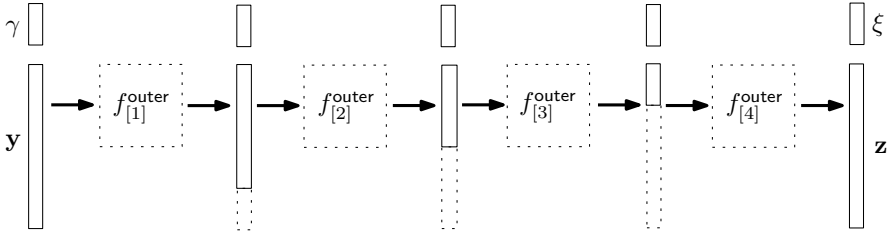}
	}
	\caption{The structure of the augmented KRnet, where the data will evolve from $\by$ to $\bz$ and the augmented vector will evolve from $\bgamma$ to $\bxi$. In this example, the data vector has been partitioned into four parts. After each iteration $f_{[i]}^{\mathsf{outer}}$, $i=1,2,3$, an extra part of the data vector will be deactivated as illustrated by the dotted line. }\label{fig:augmented_structure_KR}
\end{figure}

\subsection{Loss and the marginal PDF}
Now let us look at the loss. Let $p_{\mathsf{data}}$ correspond to the PDF for the data. By the definition of the augmented KRnet, we need to minimize the KL divergence
\begin{align*}
&D_{\mathsf{KL}}(p_{\mathsf{data}}(\by)p_{\bgamma}(\bgamma)\|p_{\bY_{\bgamma},\btheta}(\by,\bgamma))=\int  p_{\mathsf{data}}p_{\bgamma}\ln\frac{p_{\mathsf{data}}p_{\bgamma}}{p_{\bY_{\bgamma},\btheta}}d\by d\bgamma\\
=&\int p_{\mathsf{data}}\ln p_{\mathsf{data}}d\by+\int p_{\mathsf{data}}p_{\bgamma}\ln \frac{p_{\bgamma}}{p_{\bY_{\bgamma},\btheta}}d\by d\bgamma,
\end{align*}
where we use $\btheta$ to indicate the model parameters. 
Since the first term on the right-hand side is only related to the data, it is equivalent to minimize the second term, which defines the loss
\begin{equation}
{L}(p_{\bY_{\bgamma},\btheta})=\frac{1}{N}\sum_{i=1}^N\ln\frac{p_{\bgamma}(\bgamma^{(i)})}{p_{\bY_{\gamma},\btheta}(\bgamma^{(i))},\by^{(i)})}\approx\mathbb{E}_{p_{\mathsf{data}}p_{\bgamma}}\left[\ln\frac{p_{\bgamma}}{p_{\bY_{\bgamma},\btheta}}\right],
\end{equation}
where for each $\by^{(i)}$ we associate one $\bgamma^{(i)}$ sampled independently from $p_{\bgamma}$.

Finally, we look at the approximation of the marginal distribution $p_{\bY}(\by)$. By the construction of the augmented KRnet, we have
\begin{equation}\label{eqn:aug_pdf_gamma}
p_{\bY}(\by)p_{\bgamma}(\bgamma)\approx p_{\bY_{\bgamma},\btheta}(\by_{\bgamma}),\quad\forall \by_{\bgamma}.
\end{equation}
To get rid of $\bgamma$, we have at least two choices:
\begin{enumerate}
\item Integrating out $\bgamma$, we have
\begin{equation}
p_{\bY}(\by)\approx \mathbb{E}_{p_{\bgamma}}\left[\frac{p_{\bY_{\bgamma},\btheta}}{p_{\bgamma}}\right]\approx\frac{1}{N}\sum_{i=1}^N\frac{p_{\bY_{\bgamma},\btheta}(\by,\bgamma^{(i)})}{p_{\bgamma}(\bgamma^{(i)})},\quad\forall\by,
\end{equation}
where $\{\bgamma^{(i)}\}_{i=1}^N$ are sampled from $p_{\bgamma}$. If equation \eqref{eqn:aug_pdf_gamma} is well approximated, the variance of the integrand should be very small, meaning that a small sample size $N$ is sufficient.
\item Picking a certain $\gamma^*$, such that
\begin{equation}
p_{\bY}(\by)\approx p_{\bY_{\bgamma},\btheta}(\by,\bgamma^*)p^{-1}_{\bgamma}(\bgamma^*),\quad\forall \by.
\end{equation}
In particular, we may choose 
\begin{equation}
	\bgamma^*=\argmax p_{\bgamma}(\bgamma),
\end{equation}
such that
\begin{equation}
	p_{\bY}(\by)=p_{\bgamma}^{{-1}}(\bgamma^*)p_{\bZ_{\bxi}}(f_{\mathsf{aug},\btheta}(\by_{\bgamma=\bgamma^*}))|\nabla_{\by_{\bgamma=\bgamma^*}}f_{\mathsf{aug},\btheta}(\by_{\bgamma=\bgamma^*})|.
\end{equation}
If we let $p_{\bgamma}$ and $p_{\bZ_{\bxi}}$ be the standard Gaussian, i.e., 
\begin{equation}
	p_{\bgamma}=\mathcal{N}(0,\mathbf{I}),\quad p_{\bZ_{\bxi}}=\mathcal{N}(0,\mathbf{I}),
\end{equation}
we have $\bgamma^*=\mathbf{0}$ and 
\begin{equation}
	p_{\bY}(\by)\approx(2\pi)^{m/2}p_{\bZ_{\bxi}}(f_{\mathsf{aug},\btheta}(\by_{\bgamma=\bm{0}}))|\nabla_{\by_{\bgamma=\bm{0}}}f_{\mathsf{aug},\btheta}(\by_{\bgamma=\bm{0}})|.
\end{equation}
\end{enumerate}

\subsection{The complexity of the augmented KRnet}
We count the number of trainable parameters in KRnet. Let us first exclude the rotation layers and the nonlinear layers, and assume that each $f_{[k]}^{\textsf{outer}}$ has $L$ general coupling layers $f_{[k,i]}^{\textsf{inner}}$. Let $n_k$ be the number of effective dimensions for $f_{[k]}^{\textsf{outer}}$ and $\mathsf{NN}_{k,i}$ the neural network (see equation \eqref{eqn:NN}) used in $f_{[k,i]}^{\textsf{inner}}$. Assume all $\mathsf{NN}_{k,i}$ are a plain neural network with two fully connected hidden layers of width $m_k$. Let $\mathsf{NN}_{k,i}$ define a mapping from $\mathbb{R}^{n_{k,1}}$ to $\mathbb{R}^{2n_{k,2}}$ with $n_k=n_{k,1}+n_{k,2}$. The number of model parameters in $\mathsf{NN}_{k,i}$ is:
\[
m_{k}^2+2m_{k}+m_kn_k+(m_k+3)n_{k,2}.
\] 
By definition, $\mathsf{NN}_{k,i+1}$ defines a mapping from $\mathbb{R}^{n_{k,2}}$ to $\mathbb{R}^{2n_{k,1}}$ with the number of model parameters as
\[
m_{k}^2+2m_{k}+m_kn_k+(m_k+3)n_{k,1}.
\]
If we combine the two adjacent affine coupling layers, we obtain 
\[
2m_k^2+4m_k+3(m_k+1)n_k,
\]
which only depends on $m_k$ and $n_k$, and is independent of the partition introduced by the affine coupling layer. If $L$ is even, we can simply regard that $\mathsf{NN}_{k,i}$ have the same number of model parameters as
\[
N_{\mathsf{NN}_k}=m_k^2+2m_k+3(m_k+1)n_k/2.
\]
We note that the main characteristic of KRnet is that a portion of dimensions will be deactivated as $k$ increases. As $n_k$ decreases with $k$, we expect that the neural network $\mathsf{NN}_{k,i}$ may become simpler for a larger $k$. In other words, $N_{\textsf{NN},k}$ decreases as $k$ increases. A simple choice to achieve this is to decease the width of $\mathsf{NN}_{k,i}$ in terms of $k$. We let $m_{k+1}=\lceil rm_k\rceil $ with $0<r<1$.  The number of trainable parameters is  $2n_k$ for the scale and bias layer. Assume that $n=mK$. We have $n_k=(n+m)-(k-1)m$, $k=1,\ldots,K$. According to figures \ref{fig:augmented_structure_KR} and \ref{fig:structure_diagram}, we have the total number of model parameters of the augmented KRnet  as
\begin{equation}
N_{\mathsf{dof}}=\sum_{k=1}^{K}(N_{\textsf{NN}_k}L+2(K-k+2)mL)=N_{\mathsf{aKR}}L,
\end{equation}
with
\begin{equation}
N_{\mathsf{aKR}}=\sum_{k=1}^{K}(N_{\textsf{NN}_k}+2(K-k+2)m).
\end{equation}
The model complexity is mainly determined by the depth $L$ and the number $K$ for the partition of data. 

We now look at the rotation and nonlinear invertible layers. The total number of parameters from rotation layers is
\begin{equation}
\sum_{i=2}^{K}n_{k}^2=\sum_{i=2}^{K}(im)^2=\frac{mn(K+1)(2K+1)-6m^2}{6},
\end{equation}
where we assume that $n=mK$. The summation is based on two constraints: (1) only the data dimensions are rotated, and (2) at stage $K$, no rotation is needed for deactivation since it is right before the final output. The total number of parameter from nonlinear invertible layers is $nn_p$, where $n_p$ is the number of grid points for the partition of the interval $[-a,a]$, see equation \eqref{eqn:nonlinear_layer_def}. Since both  rotation layers and nonlinear invertible layers do not depend on the inner loop $f_{[k,i]}^{\mathsf{inner}}$, the portion of the DOFs from these two types of layers is in general small.

\subsection{An augmented neural ODE}
Recently the connection between Resnet and the discretization of ODE has been observed and exploited to construct deep nets subject to a certain type of recursive structure \cite{He_2015,Lu_2020,Chen_2019}. Using an ODE to describe the evolution of $\bx$ in terms of $t$, i.e.,
\begin{equation}\label{eqn:ODE}
\frac{d\bx}{dt}=\bv(\bx;\btheta),\quad \forall \bx\in[0,T],
\end{equation} 
neural ODE models the velocity field with a neural network and treats the learning process as a parameter estimation problem for the ODE model. In terms of our problem, we associate $\bx(0)$ with the data distribution, and expect the distribution at $\bx(T)$ is consistent with the prior distribution. The transformation from $\bx(0)$ to $\bx(T)$, or from $\bx(T)$ to $\bx(0)$, will be achieved by a forward or backward numerical discretization of the ODE \eqref{eqn:ODE} respectively. We note that no numerical schemes can maintain exactly the invertibility between the forward and backward integration. This implies that if we want to maintain the exact invertibility for a continuous model, we should not assume that the velocity field $\bv(\bx)$ is simply modeled by a general neural network. We intend to incorporate the structure of the augmented KRnet into the definition of the velocity field of an ODE and maintain the exact invertibility in the discretization as well.

\subsubsection{Neural ODE from an exactly invertible mapping}
We first reformulate two consecutive affine coupling layers \eqref{eqn:new_affine} as
\begin{equation}\label{eqn:affine_for_ode_one}
\left\{
\begin{array}{ll}
&\bz_1=\by_{[i],1}+\left[\by_{[i],1}\odot\bw_{1}(\by_{[i],2})+\bb_{1}(\by_{[i],2})\right]\Delta t,\\
&\bz_2=\by_{[i],2},
\end{array}
\right.
\end{equation}
and
\begin{equation}\label{eqn:affine_for_ode_two}
\left\{
\begin{array}{ll}
&\by_{[i+1],1}=\bz_{1},\\
&\by_{[i+1],2}=\bz_2+\left[\bz_{2}\odot\bw_{2}(\bz_1)+\bb_{2}(\bz_1)\right]\Delta t,
\end{array}
\right.
\end{equation}
where $\bw_i(\cdot)$ and $\bb_i(\cdot)$ take the following form
\begin{align}
\bw_i(\bx)&=e^{\balpha}\odot\tanh(\bs_i(\bx)),\\
\bb_i(\bx)&=e^{\bbeta}\odot\tanh(\bt_i(\bx)),
\end{align}
with $(\bs_i,\bt_i)= \textsf{NN}_i(\bx)$ is an neural network with input $\bx$. Compared to equations \eqref{eqn:new_affine}, we replace the constant $\alpha$ with a trainable scaling factor $e^{\balpha}$ and a constant $\Delta t$. 

Combining the two consecutive affine layers as one layer such that the whole vector gets updated, we have
\begin{equation}\label{eqn:af_2_delta_t}
\left\{
\begin{array}{rcl}
\by_{[i+1],1}&=&\by_{[i],1}+\bg_1(\by_{[i],1},\by_{[i],2})\Delta t,\\
\by_{[i+1],2}&=&\by_{[i],2}+\bg_2(\by_{[i+1],1},\by_{[i],2})\Delta t,
\end{array}
\right.
\end{equation}
where
\begin{align*}
{\bg}_1(\by_{[i],1},\by_{[i],2})&=\by_{[i],1}\odot\bw_{1}(\by_{[i],2})+\bb_{1}(\by_{[i],2})\\
{\bg}_2(\by_{[i+1],1},\by_{[i],2})&=\by_{[i],2}\odot\bw_{2}(\by_{[i+1],1})+\bb_{2}(\by_{[i+1],1})
\end{align*}
Note that
\begin{align*}
\lim_{\Delta t\rightarrow0}\frac{\by_{[i+1],1}-\by_{[i],1}}{\Delta t}&=\bg_1(\by_{[i],1},\by_{[i],2}),\\
\lim_{\Delta t\rightarrow0}\frac{\by_{[i+1],2}-\by_{[i],2}}{\Delta t}&=\lim_{\Delta t\rightarrow0}\bg_2(\by_{[i+1],1},\by_{[i],2})=\bg_2(\by_{[i],1},\by_{[i],2}).
\end{align*}
So equation \eqref{eqn:af_2_delta_t} can be regarded as an explicit one-step numerical method of a dynamical system
\begin{equation}\label{eqn:f2_ds}
\left\{
\begin{array}{rcl}
\frac{d\by_{1}}{dt}&=&\bg_1(\by_{1},\by_{2})=\by_{1}\odot\bw_{1}(\by_{2})+\bb_{1}(\by_{2}),\\
\frac{d\by_{2}}{dt}&=&\bg_2(\by_{1},\by_{2})=\by_{2}\odot\bw_{2}(\by_{1})+\bb_{2}(\by_{1}),
\end{array}
\right.
\end{equation}
where the only difference than a regular explicit Euler scheme is an updated $\by_1$ is used in the discretization of the second equation. Let $f_{\mathsf{af},1}(\cdot)$ and $f_{\mathsf{af},2}(\cdot)$ indicate the affine coupling layers given by equations \eqref{eqn:affine_for_ode_one} and \eqref{eqn:affine_for_ode_two} respectively. We define 
\[
f^i_{\mathsf{af},1,2}(\cdot)= \underbrace{(f_{\mathsf{af},1}\circ f_{\mathsf{af},2})\circ \ldots\circ (f_{\mathsf{af},1}\circ f_{\mathsf{af},2})}_\text{$i$}(\cdot)
\] 
We see that by equations \eqref{eqn:affine_for_ode_one} and \eqref{eqn:affine_for_ode_two} the mapping
\begin{equation}
\by_{[i+1]}=\by_{[i]}+(f_{\mathsf{af},1,2}-\mathrm{Id})(\by_{[i]})
\end{equation}
is invertible, where $\by_{[i]}=(\by_{[i],1},\by_{[i],2})$ and $\mathrm{Id}$ is an identity operator. In particular, the limit 
\[
\lim_{\Delta t\rightarrow0}\frac{\by_{[i+1]}-\by_{[i]}}{\Delta t}
\]
exists, which defines the dynamical system \eqref{eqn:f2_ds}.  This result can be generalized as the following lemma.
\begin{lemma}\label{lem:neural_ode_afs}
Let $\by_{[i+1]}=f_{\mathsf{af},1,2}^{m}(\by_{[i]})$ with $m\in\mathbb{N}_+$. There exists $\bg_{[m]}(\by_{[i]})\in\mathbb{R}^n$ such that
\[
\lim_{\Delta t\rightarrow0}\frac{\by_{[i+1]}-\by_{[i]}}{\Delta t}=\bg_{[m]}(\by_{[i]}),
\]
i.e., $\by_{[i]}$ can be regarded as an approximation of the ODE
\[
\dot{\by}=\bg_{[m]}(\by),
\]
if $\bg_{[m]}(\by)$ is sufficiently smooth.
\end{lemma}
\begin{proof}
We argue by induction. It is seen that it is true when $m=1$. Assume the conclusion holds for $m\leq k$. We have
\[
\by_{[i+1]}=f_{\mathsf{af},1,2}^{k+1}(\by_{[i]})=f_{\mathsf{af},1,2}\circ f_{\mathsf{af},1,2}^{k}(\by_{[i]}).
\] 
Let $\bz=f_{\mathsf{af},1,2}^{k}(\by_{[i]})$. We have
\[
\lim_{\Delta t\rightarrow0}\frac{\by_{[i+1]}-\by_{[i]}}{\Delta t}=\lim_{\Delta t\rightarrow0}\frac{\by_{[i+1]}-\bz+\bz-\by_{[i]}}{\Delta t}.
\]
According to the assumption, there exist $\bg_{[1]}(\cdot)$ and $\bg_{[k]}(\cdot)$ such that
\begin{align*}
\lim_{\Delta t\rightarrow0}\frac{\by_{[i+1]}-\bz}{\Delta t}&=\bg_{[1]}(\bz),\\
\lim_{\Delta t\rightarrow0}\frac{\bz-\by_{[i]}}{\Delta t}&=\bg_{[k]}(\by_{[i]}).
\end{align*}
We can then let
\begin{equation}\label{eqn:force_term_recursive}
\bg_{[k+1]}(\by_{[i]})=\bg_{[1]}(\by_{[i]})+\bg_{[k]}(\by_{[i]}),
\end{equation}
since
\[
\lim_{\Delta t\rightarrow0}\bz=\lim_{\Delta t\rightarrow0}f_{\mathsf{af},1,2}^{k}(\by_{[i]})=\by_{[i]}
\]
by definition. 
\end{proof}
\begin{remark}
It is seen from equation \eqref{eqn:force_term_recursive} that every time $f_{\mathsf{af},1,2}(\cdot)$ is introduced, a vector function $\bg_{[1]}(\by_{[i]};\btheta_j)$ is added to the velocity field such that
\begin{equation}
\bg_{[k]}(\by_{[i]})=\sum_{j=1}^k\bg_{[1]}(\by_{[i]};\btheta_j),
\end{equation} 
where we include the model parameters $\btheta_j$ to differentiate the $k$ functions $g_{[1]}(\by_{[i]};\btheta_j)$, $j=1,\ldots,k$.
\end{remark}
\begin{remark}\label{rmk:multi_stage_mapping}
The mapping $\by_{[i+1]}=f_{\mathsf{af},1,2}^{m}(\by_{[i]})$ can be regarded as a multi-stage process that is defined on a time interval $[0,\Delta t]$. Let 
\begin{equation}
\by_{[i+\frac{j}{m}]}=f_{\mathsf{af},1,2}\left(\by_{[i+\frac{j-1}{m}]},\btheta_{j},\Delta t\right),\quad j=1,\ldots,m,
\end{equation}
where the transform from $\by_{[i+\frac{j-1}{m}]}$ to $\by_{[i+\frac{j}{m}]}$ is achieved at stage $j$. We can then decompose $\by_{[i+1]}=f_{\mathsf{af}}^{2m}(\by_{[i]})$ as
\begin{equation}
\left\{
\begin{array}{lcl}
\by_{[i+\frac{1}{m}]}&=&f_{\mathsf{af},1,2}\left(\by_{[i]},\btheta_1,\Delta t\right)\\
\by_{[i+\frac{2}{m}]}&=&f_{\mathsf{af},1,2}\left(\by_{[i+\frac{1}{m}]},\btheta_2,\Delta t\right)\\
&\ldots&\\
\by_{[i+1]}&=&f_{\mathsf{af},1,2}\left(\by_{[i+\frac{m-1}{m}]},\btheta_m,\Delta t\right)
\end{array}
\right.
\end{equation}
Note that the following two limits exist
\[
\lim_{\Delta t\rightarrow0}\frac{\by_{[i+\frac{j}{m}]}-\by_{[i+\frac{j-1}{m}]}}{\Delta t}={\bg}_{[1]}(\by_{[i+\frac{j-1}{m}]},\btheta_j),\quad \lim_{\Delta t\rightarrow0}\by_{[i+\frac{j}{m}]}=\by_{[i]}.
\]
We then have
\[
\lim_{\Delta t\rightarrow0}\frac{\by_{[i+1]}-\by_{[i]}}{\Delta t}=\lim_{\Delta t\rightarrow0}\sum_{j=1}^m\frac{\by_{[i+\frac{j}{m}]}-\by_{[i+\frac{j-1}{m}]}}{\Delta t}=\sum_{j=1}^m{\bg}_{[1]}(\by_{[i]},\btheta_j).
\]
Compared to the multi-stage numerical schemes such as the Runge-Kutta method for the numerical approximation of ODE, we use multiple stages to achieve the exact invertibility rather than a better accuracy.   
\end{remark}

\subsubsection{Generalize the model}
We generalize the model in lemma \ref{lem:neural_ode_afs} by integrating an augmented KRnet into the definition of the velocity field. We simply consider the recursive formula defined by an augmented KRnet:
\begin{equation}
\by_{\bgamma,[i+1]}=f_{\mathsf{KRnet}}(\by_{\bgamma,[i]})=(L_S\circ f^{m_K}_{[\mathsf{af},1,2],K})\circ\ldots\circ (L_S\circ f^{m_1}_{[\mathsf{af},1,2],1})(\by_{\bgamma,[i]}),
\end{equation}
where $f_{[\mathsf{af},1,2],k}$ indicates two affine coupling layers given by equations \eqref{eqn:affine_for_ode_one} and \eqref{eqn:affine_for_ode_two} at stage $k$, and the active dimensions for $f^{m_k}_{[\mathsf{af},1,2],k}$ are defined with respect to figure \ref{fig:augmented_structure_KR}. Following remark \ref{rmk:multi_stage_mapping}, $f_{\mathsf{KRnet}}(\cdot)$ can be understood as a one-step method, where multiple stages are used to maintain the exact invertibility. In other words, the following limit exists
\[
\lim_{\Delta t\rightarrow0}\frac{\by_{\bgamma,[i+1]}-\by_{\bgamma,[i]}}{\Delta t}=\lim_{\Delta t\rightarrow0}\frac{f_{\mathsf{KRnet}}(\by_{\bgamma,[i]})-\by_{\bgamma,[i]}}{\Delta t}=\bg_{\mathsf{KRnet}}(\by_{\bgamma,[i]}),
\]  
which suggests a dynamical system
\begin{equation}\label{eqn:krnet_ds}
\frac{d\by_{\bgamma}}{dt}=\bg_{\mathsf{KRnet}}(\by_{\bgamma}).
\end{equation}
\begin{figure}	\center{
\includegraphics[width=0.45\textwidth]{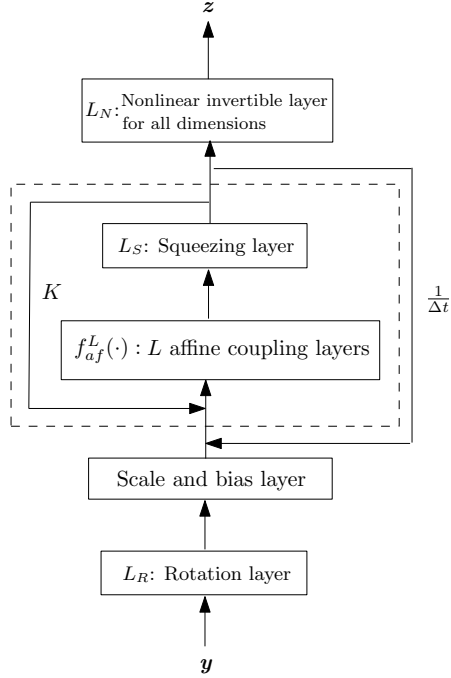}
	}
	\caption{The flow chart of an neural ODE based on an invertible recursive mapping given by KRnet. The dashed rectangle indicates the operation for each time step, where the time interval $[0,1]$ is uniformly discretized with step size $\Delta t$. }\label{fig:structure_diagram_aug_neural_ODE}
\end{figure}

\subsubsection{The adjoint method for an invertible mapping}
One difficulty of neural ODE is that as the time step size decreases the size of the computation graph for automatic differentiation may explode and exhaust the computer memory quickly. We then need to consider the adjoint method to compute the gradient for the optimizer. The adjoint method is defined with respect to an ODE, where the system is invertible. For a certain path from $\bx(0)$ to $\bx(T)$, the adjoint method needs to integrate the ODE \eqref{eqn:ODE} backwardly from $\bx(T)$ to $\bx(0)$. Once the ODE is discretized, the exact invertibility will be lost at the discrete level, implying that the adjoint method in general cannot yield the gradient up to the machine accuracy. 

In our model, we do not need to formulate the adjoint method in terms of the ODE since the exact invertibility is kept by definition. We first consider the following optimization problem
\begin{equation}\label{eqn:L_ce}
\min_{\btheta} L=-\mathbb{E}_{p_{\mathsf{data}}}\log p_{\bY}(\by;\btheta)=-\frac{1}{N}\sum_{j=1}^N\log p_{\bY}(\by^{(j)};\btheta)
\end{equation}
subject to the following constraints:
\begin{equation}
\by_{[i+1]}=F(\by_{[i]},\btheta),\quad i=0,1,\ldots,n-1,
\end{equation}
where $F$ can be regarded as an invertible mapping defined by equation \eqref{eqn:af_2_delta_t}, and the subscript $*_i$ indicates the temporal discretization. We let $\by_{[0]}=\by$ and assume $\by_{[n]}=\bz$ has a standard Gaussian distribution. From equation \eqref{eqn:pdf_model}, we have
\begin{align}
\log p_{\bY}(\by)=&\log p_{\bZ}(\bz)+\sum_{i=0}^{n-1}\log|\det\nabla_{\by_{[i]}}\by_{[i+1]}|\nonumber\\
=&\log p_{\bZ}(\bz)+\sum_{i=0}^{n-1}g_{[i]}(\by_{[i]},\btheta),
\end{align}
where $g_{[i]}(\by_{[i]},\btheta)$ can be explicitly computed by the definition of the affine coupling layer. For simplicity, we only consider one data point and ignore the superscript such that
\begin{equation}\label{eqn:one_point_CE_loss}
L=-\log p_{\bY}(\by;\btheta).
\end{equation}
To compute $\nabla_{\btheta}L$, we consider the following Lagrangian:
\begin{equation}
\mathcal{L}=-\log p_{\bZ}(\bz)-\sum_{i=0}^{n-1}g_{[i]}(\by_{[i]},\btheta)-\sum_{i=0}^{n-1}\blambda_{[i]}^\mathsf{T}(\by_{[i+1]}-F_{[i]}(\by_{[i]},\btheta)).
\end{equation}
The key idea of the adjoint method is to choose appropriate Lagrange multipliers $\blambda_{[i]}$ such that the computation of the gradient is convenient. We have
\begin{align*}
\nabla_{\btheta}\mathcal{L}=&(\nabla_{\btheta}\bz)^\mathsf{T}(-\nabla_{\bz}\log p_{\bZ}(\bz))-\sum_{i=1}^{n-1}\left((\nabla_{\btheta}\by_{[i]})^\mathsf{T}\nabla_{\by_{[i]}}g_{[i]}+\nabla_{\btheta}g_{[i]}\right)-\nabla_{\btheta}g_{[0]}\\
&-\sum_{i=1}^{n-1}\left(\nabla_{\btheta}\by_{[i+1]}-\nabla_{\by_{[i]}}F_{[i]}\nabla_{\btheta}\by_{[i]}-\nabla_{\btheta}F_{[i]}\right)^\mathsf{T}\blambda_{[i]}\\
&-\left(\nabla_{\btheta}\by_{[1]}-\nabla_{\btheta}F_{[0]}\right)^\mathsf{T}\blambda_{[0]}\\
=&(\nabla_{\btheta}\bz)^\mathsf{T}(-\nabla_{\bz}\log p_{\bZ}(\bz)-\blambda_{[n-1]})-\sum_{i=0}^{n-1}\nabla_{\btheta}g_{[i]}+\sum_{i=0}^{n-1}(\nabla_{\btheta}F_{[i]})^\mathsf{T}\blambda_{[i]}\\
&-\sum_{i=1}^{n-1}(\nabla_{\btheta}\by_{[i]})^\mathsf{T}(\blambda_{[i-1]}-(\nabla_{\by_{[i]}}F_{[i]})^\mathsf{T}\blambda_{[i]}+\nabla_{\by_{[i]}}g_{[i]}).
\end{align*}
We let
\begin{equation}
\left\{
\begin{array}{lcl}
\blambda_{[i-1]}&=&(\nabla_{\by_{[i]}}F_{[i]})^\mathsf{T}\blambda_{[i]}-\nabla_{\by_{[i]}}g_{[i]},\quad i=n-1,\ldots,1\\
\blambda_{[n-1]}&=&-\nabla_{\bz}\log p_{\bZ}(\bz)
\end{array}
\right.
\end{equation}
and obtain that
\begin{equation}
\nabla_{\btheta}\mathcal{L}=\nabla_{\btheta}L=\sum_{i=0}^{n-1}\left((\nabla_{\btheta}F_{[i]})^\mathsf{T}\blambda_{[i]}-\nabla_{\btheta}g_{[i]}\right).
\end{equation}

Except for the neural ODE layer, we may add other types of layers into the model. The following lemma provides a more general result for the adjoint method in terms of an invertible mapping:
\begin{lemma}\label{lem:adjoint_invertible_mapping}
Consider a general invertible mapping:
\begin{equation}
\by_{[i+1]}=F_{[i]}(\by_{[i]};\btheta_{[i]}),\quad \by_{[0]}=\by,\quad \by_{[n]}=\bz, \quad i=0,\ldots,n-1,
\end{equation}
where we let
\[
g_{[i]}(\by_{[i]},\btheta_{[i]})=\log|\det\nabla_{\by_{[i]}}\by_{[i+1]}|.
\]
Assume that the loss $L$ is given by equation \eqref{eqn:one_point_CE_loss}. The following two sequences $\blambda_{[i]}$ and $\by_{[i]}$ can be computed backwardly:
\begin{equation}
\left\{
\begin{array}{l}
\blambda_{[i-1]}=(\nabla_{\by_{[i]}}F_{[i]})^\mathsf{T}\blambda_{[i]}-\nabla_{\by_{[i]}}g_{[i]},\quad i=n-1,\ldots,1\\
\blambda_{n-1}=-\nabla_{\bz}\log p_{\bZ}(\bz),\\
\by_{[i]}=F^{-1}_{[i]}(\by_{[i+1]},\btheta_{[i]}),\quad i=n-1,\ldots,0.
\end{array}
\right.
\end{equation}
We have
\begin{equation}
\nabla_{\btheta_{[i]}}L=(\nabla_{\btheta_{[i]}}F_{[i]})^{\mathsf{T}}\blambda_{[i]}-\nabla_{\btheta_{[i]}}g_{[i]},\quad i=0,1,\ldots,n-1.
\end{equation}
If $\btheta_{[i]}=\btheta$ for $i\in\mathcal{I}\subset\{0,1,\ldots,n-1\}$, we have
\begin{equation}
\nabla_{\btheta}L=\sum_{i\in\mathcal{I}}\left((\nabla_{\btheta}F_{[i]})^{\mathsf{T}}\blambda_{[i]}-\nabla_{\btheta}g_{[i]}\right).
\end{equation}
\end{lemma}
\begin{proof}
We consider the Lagrangian
\begin{equation}
\mathcal{L}=-\log p_{\bZ}(\bz)-\sum_{i=0}^{n-1}g_{[i]}(\by_{[i]},\btheta_i)-\sum_{i=0}^{n-1}\blambda_{[i]}^\mathsf{T}(\by_{[i+1]}-F_{[i]}(\by_{[i]},\btheta_{[i]})).
\end{equation}
For $\btheta_{[k]}$, $0\leq k <  n-1$, we have
\begin{align*}
\nabla_{\btheta_{[k]}}\mathcal{L}=&(\nabla_{\btheta_{[k]}}\bz)^\mathsf{T}(-\nabla_{\bz}\log p_{\bZ}(\bz))-\sum_{i=k+1}^{n-1}(\nabla_{\btheta_{[k]}}\by_{[i]})^\mathsf{T}\nabla_{\by_{[i]}}g_{[i]}-\nabla_{\btheta_{[k]}}g_{[k]}\\
&-\sum_{i=k+1}^{n-1}\left(\nabla_{\btheta_{[k]}}\by_{[i+1]}-\nabla_{\by_{[i]}}F_{[i]}(\by_{[i]},\btheta_{[i]})\nabla_{\btheta_{[k]}}\by_{[i]}\right)^{\mathsf{T}}\blambda_{[i]}\\
&-(\nabla_{\btheta_{[k]}}\by_{[k+1]}-\nabla_{\btheta_{[k]}}F_{[k]}(\by_{[k]},\btheta_{[k]}))^{\mathsf{T}}\blambda_{[k]}\\
=&(\nabla_{\btheta_{[k]}}\bz)^\mathsf{T}(-\nabla_{\bz}\log p_{\bZ}(\bz)-\blambda_{[n-1]})\\
&-\sum_{i=k+1}^{n-1}(\nabla_{\btheta_{[k]}}\by_{[i]})^\mathsf{T}(\nabla_{\by_{[i]}}g_{[i]}-(\nabla_{\by_{[i]}}F_{[i]}(\by_{[i]},\btheta_{[i]}))^{\mathsf{T}}\blambda_{[i]}+\blambda_{[i-1]})\\
&+(\nabla_{\btheta_{[k]}}F_{[k]}(\by_{[k]},\btheta_{[k]}))^{\mathsf{T}}\blambda_{[k]}-\nabla_{\btheta_{[k]}}g_{[k]}.
\end{align*}
If $k=n-1$, we have
\begin{align*}
\nabla_{\btheta_{[k]}}\mathcal{L}
=&(\nabla_{\btheta_{[n-1]}}\bz)^\mathsf{T}(-\nabla_{\bz}\log p_{\bZ}(\bz)-\blambda_{[n-1]})\\
&+(\nabla_{\btheta_{[n-1]}}F_{[n-1]}(\by_{[n-1]},\btheta_{[n-1]}))^{\mathsf{T}}\blambda_{[n-1]}-\nabla_{\btheta_{[n-1]}}g_{[n-1]}.
\end{align*}
Letting
\begin{equation}\label{eqn:recursive_theta_k}
\left\{
\begin{array}{lcl}
\blambda_{[i-1]}&=&(\nabla_{\by_{[i]}}F_{[i]})^{\mathsf{T}}\blambda_{[i]}-\nabla_{\by_{[i]}}g_{[i]},\quad i=n-1,\ldots,k+1,\\
\blambda_{[n-1]}&=&-\nabla_{\bz}\log p_{\bZ}(\bz),
\end{array}
\right.
\end{equation}
we have
\begin{equation}
\nabla_{\btheta_{[k]}}\mathcal{L}=(\nabla_{\btheta_{[k]}}F_{[k]}(\by_{[k]},\btheta_{[k]}))^{\mathsf{T}}\blambda_{[k]}-\nabla_{\btheta_{[k]}}g_{[k]}.
\end{equation}
Since the recursive formula \eqref{eqn:recursive_theta_k} holds for any $0\leq k\leq n-1$, we obtain the conclusion. 
\end{proof}
\begin{remark}
For the gradient of the cross entropy \eqref{eqn:L_ce} defined by $N$ data points, we need to collect the contributions from all data points to the gradient using Lemma \ref{lem:adjoint_invertible_mapping}. We have
\begin{equation}
\nabla_{\btheta}L=\frac{1}{N}\sum_{i=1}^N\sum_{j\in\mathcal{I}}\left((\nabla_{\btheta}F_{[j]}(\by_{[j]}^{(i)},\btheta_{[j]}))^{\mathsf{T}}\blambda_{[j]}^{(i)}-\nabla_{\btheta}g_{[j]}(\by_{[j]}^{(i)},\btheta_{[j]})\right),
\end{equation}
where $\btheta_{[j]}=\btheta$ for $j\in \mathcal{I}$, and the superscript $*^{(i)}$ indicates each data point.
\end{remark}

\subsection{A summary of the main features of KRnet}
To this end, we we have developed various techniques that either improve the performance of KRnet as a discrete model or reformulate it as a continuous model. We summarize some useful features of KRnet as follows:
\begin{enumerate}
\item \emph{The Knothe-Rosenblatt rearrangement} defines the main structure of the KRnet for both the discrete and continuous models. 
\item \emph{The rotation layer} provides a mechanism, which is similar to the principle component analysis, to pick a certain set of dimensions to deactivate.
\item \emph{The nonlinear layer} provides a much larger family of prior distributions than the commonly used standard Gaussian distributions through a component-wise nonlinear transformation. 
\item \emph{The augmented dimensions} provide a buffer zone for the data dimensions to exchange nonlinear information more effectively.
\item \emph{The KRnet\_ODE} integrates KRnet into a continuous model as neural ODE while the exact invertibility is maintained. The adjoint method can be formulated with respect to the discrete model instead of the continuous one such that the gradient of the loss can be computed exactly. 
\end{enumerate}
The features 1-4 can be coupled to improve the performance of a discrete model; The features 1, 4 and 5 can be coupled to improve the performance of a continuous model. 

\subsection{Density estimation and approximation via KRnet}
The developed KRnets may be used to construct a PDF model for both density estimation and approximation. For density estimation, we assume that the empirical distribution $p_{\mathsf{data}}(\by)$ is given, and for density estimation, we assume that the unnormalized PDF $\hat{p}_{\bY}(\by)=Cp_{\mathsf{ref},\bY}(\by)$ is given, where $p_{\mathsf{ref},\bY}$ is the true PDF and $C$ is an unknown constant. For both density estimation and approximation, we can use the Kullback-Leibler (KL) divergence to minimize the difference between the given distribution and the PDF model $p_{\mathsf{KRnet}}(\by)$ based on KRnet. 

For density estimation, we consider the KL divergence
\begin{equation}
D_{\mathsf{KL}}(p_{\mathsf{data}}\|p_{\mathsf{KRnet},\bY})=h(p_{\mathsf{data}},p_{\mathsf{KRnet},\bY})-h(p_{\mathsf{data}}),
\end{equation}
where the first term on the right-hand side is the differential cross entropy of $p_{\mathsf{KRnet},\bY}$ relative to $p_{\mathsf{data}}$, and the second term is the differential entropy of $p_{\mathsf{data}}$. Since $h(p_{\mathsf{data}})$ is independent of $p_{\mathsf{KRnet},\bY}$, minimizing the KL divergence $D_{\mathsf{KL}}(p_{\mathsf{data}}\|p_{\mathsf{KRnet},\bY})$ is equivalent to minimizing the differential cross entropy $h(p_{\mathsf{data}},p_{\mathsf{KRnet},\bY})$, which is also equivalent to maximizing the  likelihood.

For density approximation, we consider the KL divergence 
\begin{align}
D_{\mathsf{KL}}(p_{\mathsf{KRnet},\bY}\|p_{\mathsf{ref},\bY})&=\mathbb{E}_{p_{\mathsf{KRnet},\bY}}\left[\ln\frac{p_{\mathsf{KRnet},\bY}}{\hat{p}_{\bY}}d\by\right]+\ln C,\nonumber\\
&\approx\frac{1}{N}\sum_{i=1}^N\ln\frac{p_{\mathsf{KRnet},\bY}(\by^{(i)})}{\hat{p}_{\bY}(\by^{(i)})},
\end{align}
where $\{\by^{(i)}\}_{i=1}^N$ are samples from $p_{\mathsf{KRnet}}$. It is seen that we only need to minimize the first term on the right-hand side and the unknown constant $C$ does not affect the optimization. In contrast to the density estimation, we use the relative entropy from $p_{\mathsf{ref},\bY}$ to $p_{\mathsf{KRnet},\bY}$ to avoid the integration in terms of $p_{\mathsf{ref},\bY}$, where the integration with respect to $p_{\mathsf{KRnet}}$ can be easily approximated by the Monte Carlo method thanks to the generative model. For the augmented KRnet, we may consider the following KL divergence:
\begin{equation}\label{eqn:KL_for_approx}
	D_{\mathsf{KL}}(p_{\mathsf{KRnet\_aug},\bY_{\bgamma}}(\bY,\bgamma)\|p_{\mathsf{ref},\bY}p_{\bgamma}),
\end{equation}
where $p_{\mathsf{KRnet\_aug},\bY_{\bgamma}}(\bY,\bgamma)$ is the joint PDF of $\bY$ and $\bgamma$ induced by the augmented KRnet. 

\section{Numerical examples}\label{sec:num}
In this section we present some numerical experiments including one-, two-, four- and eight-dimensional problems, where PDFs with different types of support are considered. All the models have been trained with ADAM method subject to a fixed learning rate 0.001 \cite{ADAM_2017}. If no additional clarification is given, the neural networks \eqref{eqn:NN} for the affine coupling layer always have two fully-connected hidden layers
of 24 neurons. When the nonlinear invertible layers (see equation \eqref{eqn:nonlinear_layer_def}) are needed, the interval $[-20,20]$, i.e., $a=20$, is discretized to 32 elements, and $\beta=10^{-10}$. The elements are nonuniform, where the element size increases from the middle to both sides with a ratio 1.15.

\subsection{The augmented ODE model for the approximation of 1d PDFs}
The simplest case of equation \eqref{eqn:krnet_ds} includes one data dimension and one augmented dimension, which takes the following form:
\begin{equation}\label{eqn:f2_ds_KRnet_1d}
\left\{
\begin{array}{rclrcl}
\dot{\gamma}&=&v_1(\gamma,y)&=&\gamma w_{1}(y)+b_{1}(y),\\
\dot{y}&=&v_2(\gamma,y)&=&y w_{2}(\gamma)+b_{2}(\gamma),
\end{array}
\right.
\end{equation}
subject to the constraint
\begin{equation}\label{eqn:bc_1d}
\rho_{t=0}(\gamma,y)=p(\gamma)f(y),\quad \rho_{t=1}(\gamma,y)=p(\gamma)p(y),
\end{equation}
where $p(\cdot)$ is a standard Gaussian PDF and $f(\cdot)$ an arbitrary PDF. We know that $\rho$ satisfies the Liouville equation:
\begin{equation}
\partial_t\rho+\nabla\cdot(\rho\bv)=0,
\end{equation}
from which we have
\begin{equation}
\partial_t\ln\rho=\frac{1}{\rho}\partial_t\rho=-\frac{1}{\rho}(\rho\nabla\cdot\bv+\bv\cdot\nabla\rho)=-\nabla\cdot\bv-\bv\cdot\nabla\ln\rho,
\end{equation}
i.e.,
\begin{equation}\label{eqn:ln_rho_1d_aug}
\frac{d\ln\rho}{dt}=-\nabla\cdot\bv=-w_1(y(t))-w_2(\gamma(t)),
\end{equation}
subject to the boundary conditions \eqref{eqn:bc_1d}. Due to the exact invertibility, the right-hand side of equation \eqref{eqn:ln_rho_1d_aug} is given by two functions in terms of $y$ and $\gamma$ respectively. However, according to equation \eqref{eqn:f2_ds_KRnet_1d}, $y(t)$ and $\gamma(t)$ depend on each other. In terms of equation \eqref{eqn:ln_rho_1d_aug} and the boundary conditions \eqref{eqn:bc_1d}, $w_1(y)$, $w_2(\gamma)$, $b_1(y)$ and $b_2(\gamma)$ need to be chosen such that
\begin{equation}\label{eqn:ln_rho_1d_constraint}
	\ln p(\gamma(1))p(y(1))=\ln p(\gamma(0))f(y(0))-\int_0^1(w_1(y(t))+w_2(\gamma(t)))dt.
\end{equation}
Let us look at a simple case, where $w_1(\cdot)=w_2(\cdot)=0$.  In other words, the dynamics given by \eqref{eqn:f2_ds_KRnet_1d_simple} preserves volume. We have the ODE as 
\begin{equation}\label{eqn:f2_ds_KRnet_1d_simple}
\left\{
\begin{array}{rcl}
\dot{\gamma}&=&b_{1}(y),\\
\dot{y}&=&b_{2}(\gamma).
\end{array}
\right.
\end{equation}
Let $\hat{b}_1'(y)=b_1(y)$ and $\hat{b}_2'(\gamma)=b_2(\gamma)$. We then have
\begin{equation}\label{eqn:1st_integral_1d}
\hat{b}_2(\gamma(t))=\hat{b}_1(y(t))+C,
\end{equation}
where $C=\hat{b}_2(\gamma(0))-\hat{b}_1(y(0))$ is determined by the initial condition. Due to the first integral \eqref{eqn:1st_integral_1d}, we expect that both $b_1(y)$ and $b_2(\gamma)$ are complex enough for a good approximation. For example, if we simply let $b_2(\gamma)=1$, we have 
\[
y(t)=y(0)+t,\quad \gamma(t)=\hat{b}_1(y(t))-\hat{b}_1(y(0))-\gamma(0).
\]
We then model $b_1(y)$ such that
\[
p\left(y(0)+1\right)p\left(\gamma(0)+\hat{b}_1(y(0)+1)-\hat{b}_1(y(0))\right)\approx p(\gamma(0))f(y(0)).
\]
It is easy to see that no matter how complex $b_1(y)$ is the above approximation may not good enough since $y(t)$ and $\gamma(t)$ cannot be independent for the case that $b_2(\gamma)=1$. However, we should note $y(t)$ and $\gamma(t)$ may be independent of each other if they both depend on $y(0)$ and $\gamma(0)$ in a certain way. One example is the Box–Muller transform, which maps two independent uniform random variables to two independent Gaussian random variables through an invertible mapping. So both $b_1(y)$ and $b_2(\gamma)$ need to be complex enough. Furthermore, when $w_1(y)$ and $w_2(\gamma)$ are included into the model, the modeling capability will be improved further. 

Since $\gamma$ corresponds to an augmented dimension, equation \eqref{eqn:f2_ds_KRnet_1d} can be regarded as a neural ODE for the approximation of an arbitrary PDF $f(y)$. We now look at how well model \eqref{eqn:f2_ds_KRnet_1d} can evolve a standard Gaussian distribution $p(y)$ to an arbitrary distribution $f(y)$. We will consider four cases, where the support of $f(y)$ is $(-\infty,\infty)$, $(0,\infty)$, $[-1,1]$ and $[-1.5,-0.5]\cup[0.5,1.5]$, respectively. The training set has $3.2\times10^5$ samples. The Adams method is subject to 4 minibatches. Let us refer to model \eqref{eqn:f2_ds_KRnet_1d} as augmented KRnet\_ODE. We will compare its performance to the augmented KRnet. For the neural ODE, we consider a uniform temporal mesh with $\Delta t=0.1$. Both the augmented KRnet and KRnet\_ODE are defined by a sequence $f_{\mathsf{af}}^L(\cdot)$ of affine coupling layers between $y$ and $\gamma$, where $L$ is the number of affine coupling layers. In the augmented KRnet, $f_{\mathsf{af}}^L(\cdot)$ will achieve the whole transformation from data distribution to the prior distribution while in the augmented KRnet\_ODE, $f_{\mathsf{af}}^L(\cdot)$ only implements the transformation for one time step. Note that the definition of $f_{\mathsf{af}}(\cdot)$ for the KRnet is slightly different than that for the KRnet\_ODE (see equations \eqref{eqn:new_affine} and \eqref{eqn:affine_for_ode_two}).

The prior distribution is always the standard Gaussian no matter that the target distribution has a compact support or not. When the model $p_{\bY_{\bgamma}}$ converges to $f(y)p(\gamma)$, the loss function is
\begin{align*}
\mathbb{E}_{f(y)p(\gamma)}\left[\ln\frac{p(\gamma)}{p_{\bY_{\bgamma}}}\right]\rightarrow-\mathbb{E}_{f(y)p(\gamma)}\left[\ln f(y)\right]=-\mathbb{E}_{f(y)}\ln f(y),
\end{align*}
which is the differential entropy $h(f)$ of $f(y)$. We then define a relative error
\begin{equation}\label{eqn:ce_error}
\delta=\frac{|L-h(f)|}{h(f)}
\end{equation}
to measure the quality of the corresponding PDF model. We consider the following cases:

\noindent\textbf{Case (i)}: $f(y)$ is Logistic distribution on $(-\infty,\infty)$. 
Consider the logistic distribution with the location parameter $\mu=0$ and the scale parameter $s=2$. The differential entropy is $h(f)=2.0+\ln(2.0)$. The relative errors for this case are plotted in the left plot of figure \ref{fig:log_log}. It is seen that $L=2$ works well for both KRnet and KRnet\_ODE. The high oscillations are due to the uncertainty from data since the loss function is an approximation of the differential entropy given by the Monte Carlo method.  

\noindent\textbf{Case (ii)}: $f(y)$ is Lognormal distribution on $(0,\infty)$.
The lognormal distribution is given by the exponential function of a standard normal random variable. The differential entropy is $\ln(2\pi)/2+1/2$. The relative errors for this case are plotted in the right plot of figure \ref{fig:log_log}. Since the positive densities on $(-\infty,\infty)$ needs to be mapped to $(0,\infty)$, the transformation is more demanding than the previous case. When $L=2$, the KRnet\_ODE has a slightly smaller error than the KRnet. When $L=4$, both models have an error that is comparable to the uncertainty from data. 

\noindent\textbf{Case (iii)}: $f(y)$ is uniform on $[-1,1]$. The differential entropy for the uniform distribution is $\ln(2)$. For this case, the positive densities on $(-\infty,\infty)$ needs to be mapped to $[-1,1]$. As $L$ increases, the performance of both KRnet and KRnet\_ODE improves. It appears that the KRnet is more effective to reduce the loss while the KRnet\_ODE is more robust. It is seen that the error given by KRnet with $L=4$ is comparable to the error given by KRnet\_ODE with $K=8$. When $L=2$, it takes KRnet a long time to find a good local minimizer. 

\noindent\textbf{Case (iv)}: $f(y)$ is uniform on $[-1.5,0.5]\cup[0.5,1.5]$.
Compared to the previous uniform distribution, similar behavior is observed for both KRnet and KRnet\_ODE except that the error is larger for the same configuration due to the more demanding requirements on the transformation. We plot some approximate PDFs in figure \ref{fig:uniform_hole_1d_pdf} for this case and the lognormal distribution in case (ii). It is seen that the KRnet handles  discontinuities slightly better than the continuous flow defined by an ODE. 

Note that for all four cases, we map the prior Gaussian distribution defined on $(-\infty,\infty)$ to the data distribution whether the data are subject to a compact support or not. Both augmented KRnet and augmented KRnet\_ODE demonstrate effectiveness and flexibility for the density estimation.  Of course, we can integrate other techniques such as regularization and data preprocessing whenever necessary. For example, if the data are defined on a compact support, say $[\delta,1-\delta]$ with $\delta > 0$, we may use the Logistic transformation
\begin{equation}
y=\frac{s}{2}\log\frac{x}{1-x},\quad x=\frac{1}{2}(\tanh(x/s)+1)
\end{equation}
to map $x\in(0,1)$ to $y\in(-\infty,\infty)$ such that the data distribution and the prior distribution have the same support. The results of such a strategy are plotted in figure \ref{fig:uniform_hole_1d_mapped_pdf}. It is seen that the transition of KRnet at discontinuities is much sharper than that of KRnet\_ODE. 

\begin{figure}	
\center{\includegraphics[width=0.49\textwidth]{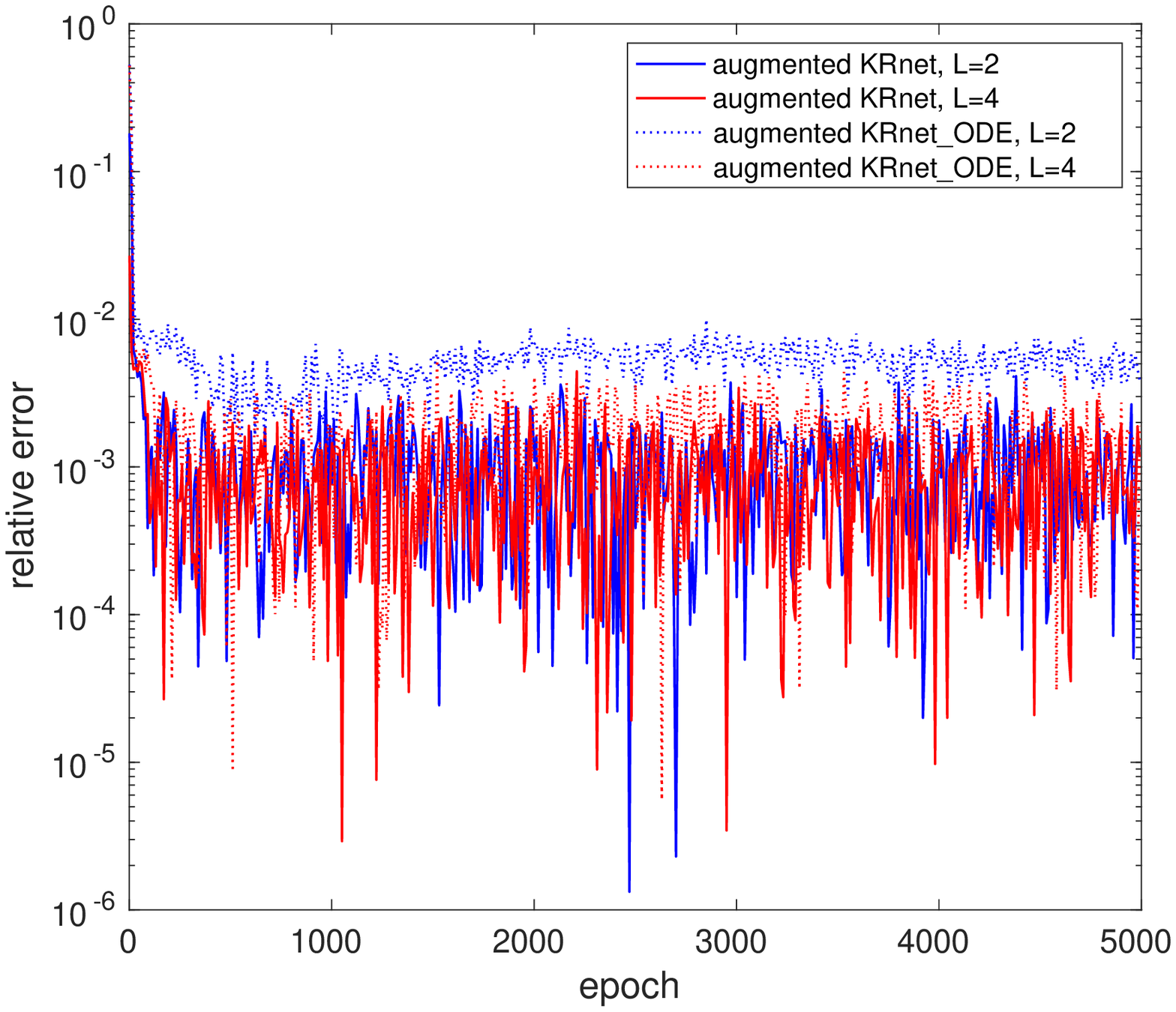}
\includegraphics[width=0.49\textwidth]{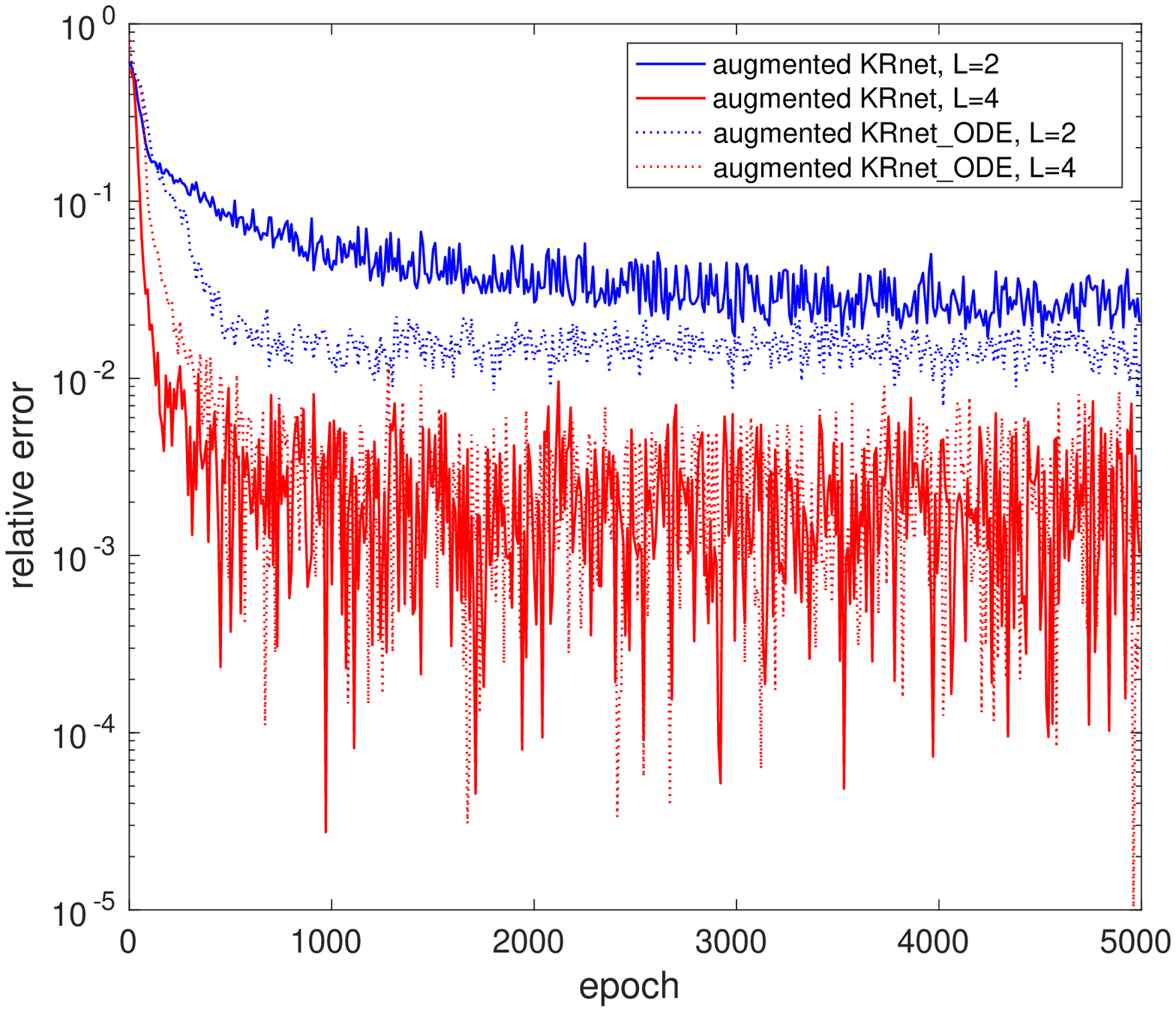}}
\caption{Compare the convergence behavior of KRnet\_aug and KRnet\_ODE. }\label{fig:log_log}
\end{figure}
\begin{figure}	
\center{\includegraphics[width=0.49\textwidth]{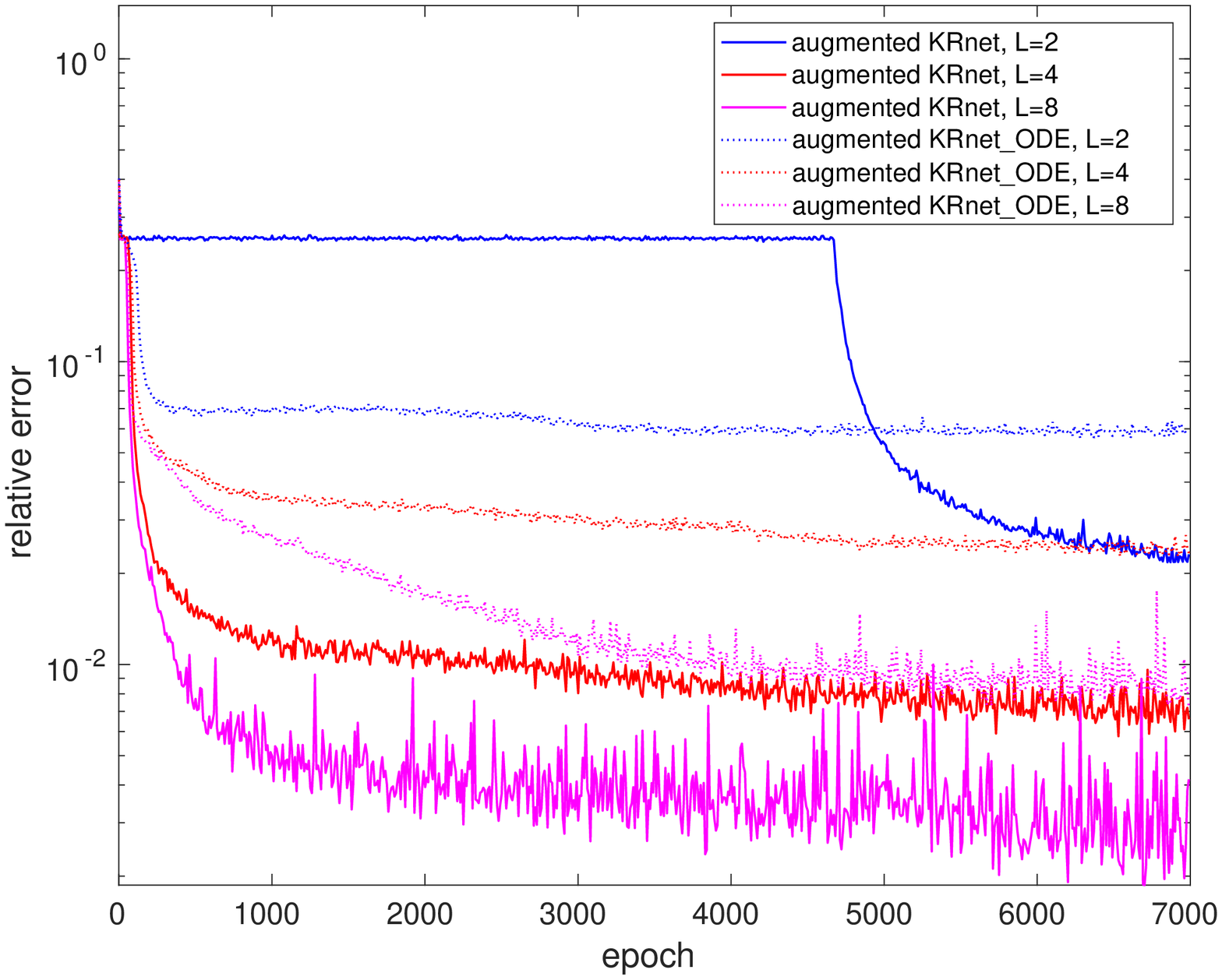}
\includegraphics[width=0.49\textwidth]{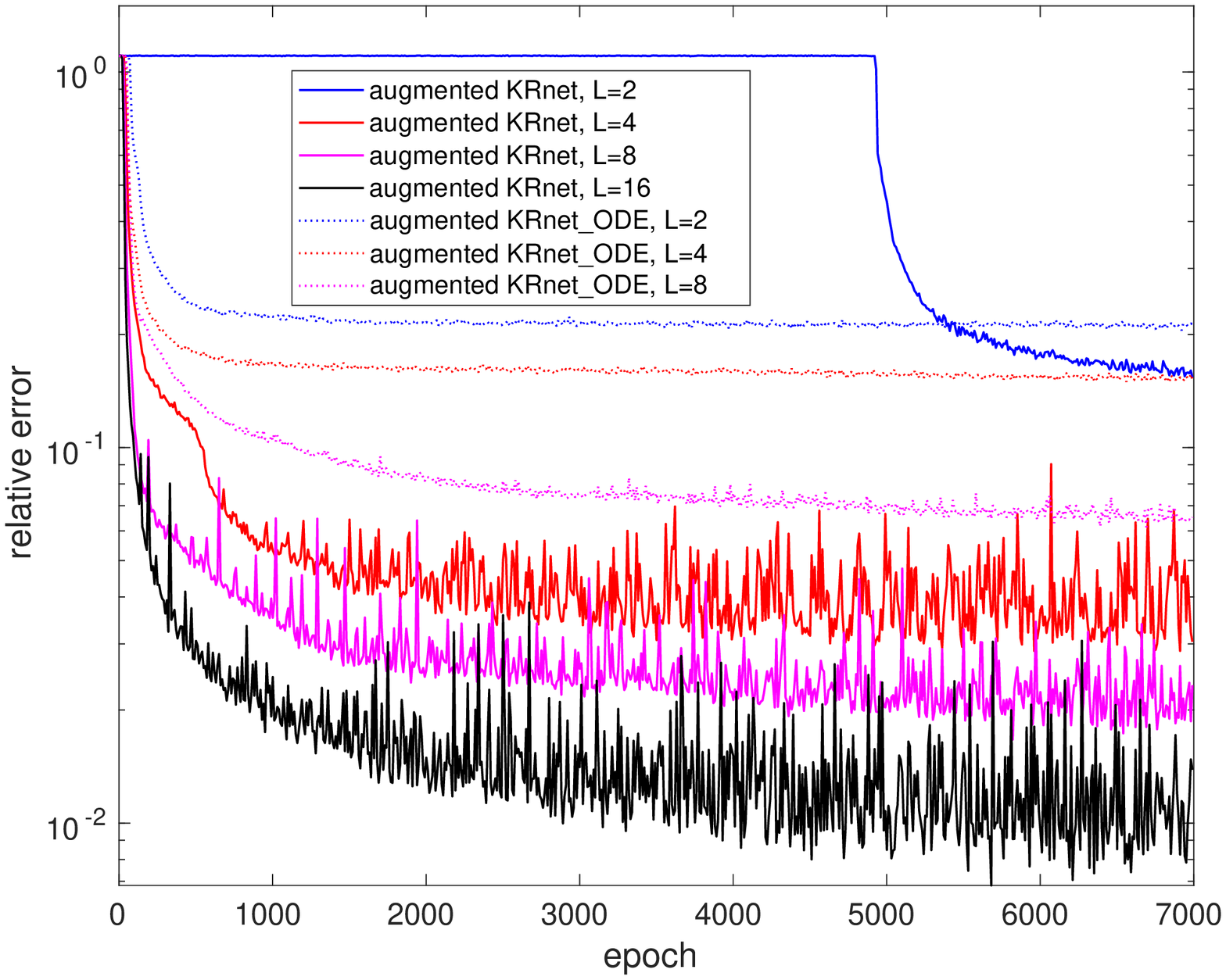}}
\caption{Compare the convergence behavior of KRnet\_aug and KRnet\_ODE. }\label{fig:u_u}
\end{figure}
\begin{figure}	
\center{
\includegraphics[width=0.49\textwidth]{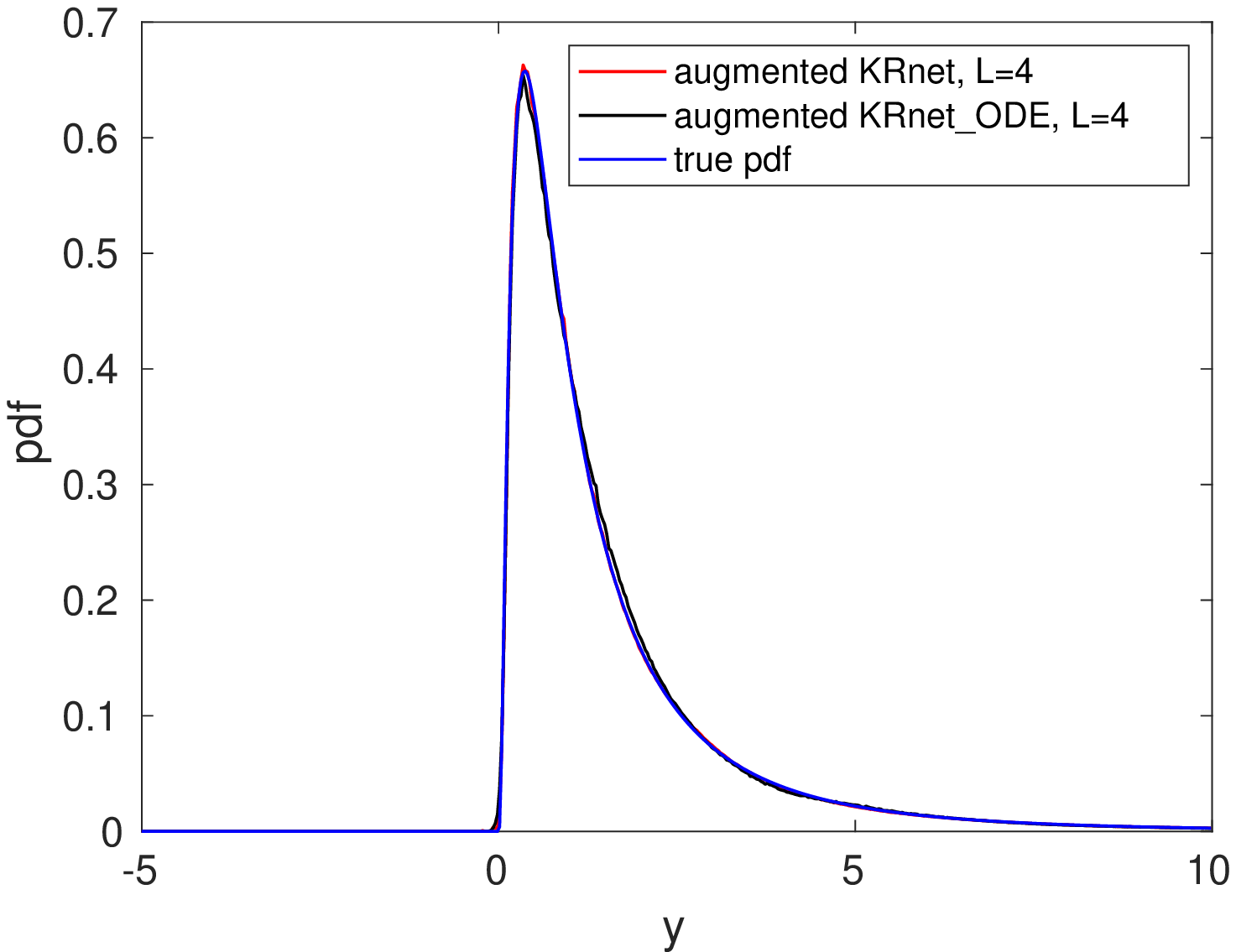}
\includegraphics[width=0.49\textwidth]{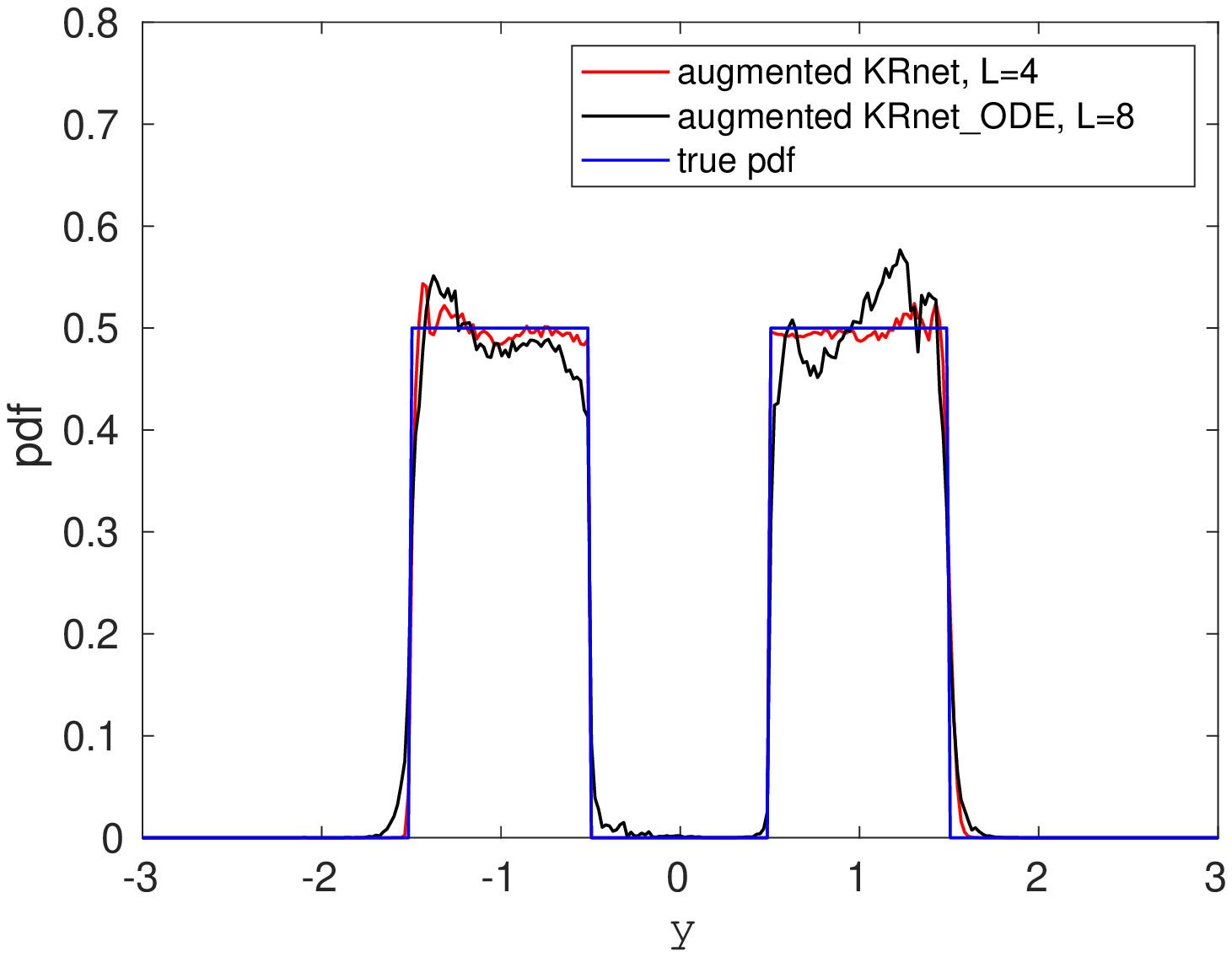}}
\caption{The approximated PDFs for the lognormal distribution and the uniform distribution with a hole. }\label{fig:uniform_hole_1d_pdf}
\end{figure}
\begin{figure}	
\center{\includegraphics[width=0.6\textwidth]{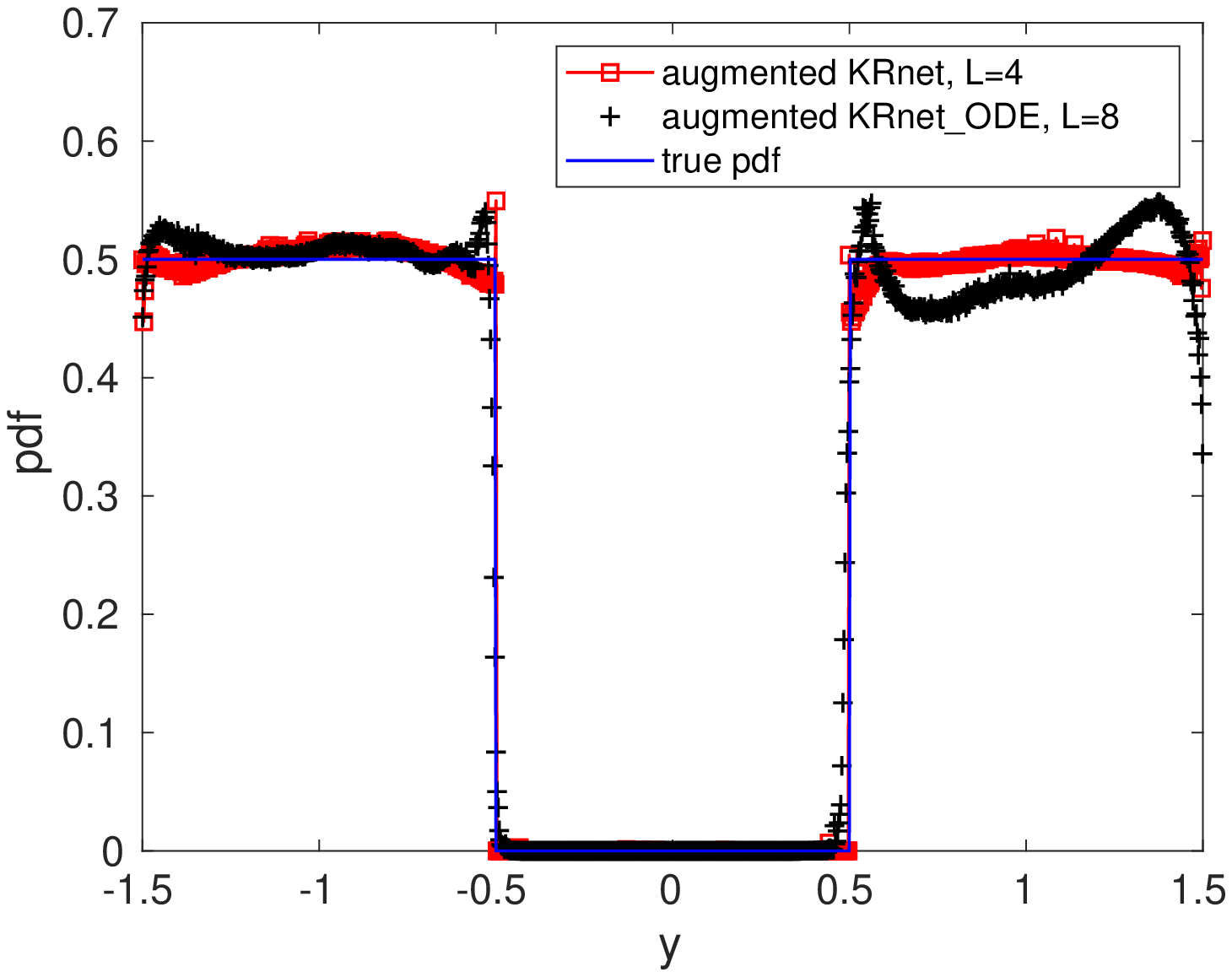}}
\caption{The approximated PDFs for the lognormal distribution and the uniform distribution with a hole. }\label{fig:uniform_hole_1d_mapped_pdf}
\end{figure}
\begin{table}
	\caption{Errors of some KRnet-based models for the density estimation of samples from the mixture of Gaussians \eqref{eqn:mixture_of_G}. In affine coupling layers, the neural network \eqref{eqn:NN} has two dense hidden layers, each of which has 24 neurons, and this number decays at a ratio $r=0.9$ in terms of the index $k$ of the outer loop of KRnet.\label{tbl:compare_KRnets}}
	\centering
	\begin{tabular}{|c|c|c|c|c|c|}
		\hline
		& {\small KRnet} & {\small KRnet\_aug} & {\small KRnet\_aug\_R\&N} & 
		{\small KRnet\_R\&N} & {\small KRnet\_ODE} \\ \hline
		{\small $L=2$}:&6.96e-2&1.02e-1&4.52e-2&1.50e-2&2.93e-2\\ \hline
		{\small $L=4$}:&1.74e-2&8.47e-3&1.29e-3&2.56e-3&1.67e-2\\ \hline
		{\small $L=6$}:&5.46e-3&1.53e-3&6.79e-4&1.56e-3&1.02e-2\\ \hline
	\end{tabular}
\end{table}  
\begin{figure}	
	\center{\includegraphics[width=0.99\textwidth]{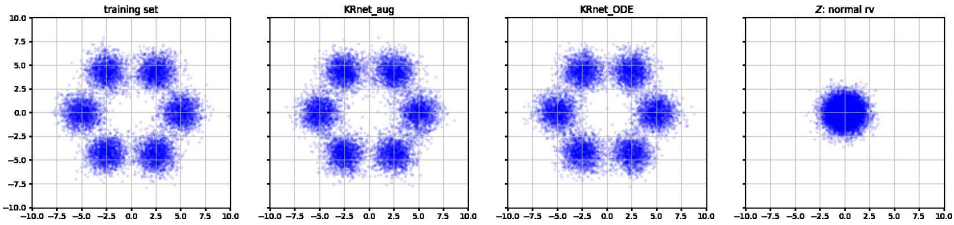}}
	\caption{Data distributions given by the training set, the augmented KRnet, and the KRnet-based neural ODE, where the prior Gaussian distribution has been mapped to the mixture of Gaussians \eqref{eqn:mixture_of_G}. $L=6$. }\label{fig:2d_mix_G_samples}
\end{figure}
\subsection{Two-dimensional mixture of Gaussians}
We consider a mixture of Gaussians 
\begin{equation}\label{eqn:mixture_of_G}
p_{\bY}(\by)=\frac{1}{6}\sum_{i=1}^6\mathcal{N}(\by_i,\mbI),
\end{equation}
where $\by_i=(5\cos\frac{i\pi}{3},5\sin\frac{i\pi}{3})$. We have six standard Gaussians uniformly located on a circle of radius 5. We examine and compare the following modeling techniques: 
\begin{itemize}
\item KRnet: This KRnet only keeps the triangular structure inspired by the K-R rearrangement. For two-dimensional problems, KRnet is consistent with the real NVP . 
\item KRnet\_aug: One augmented dimension is added to KRnet.
\item KRnet\_R\&N: The rotation layers and the nonlinear invertible layer are switched on for KRnet.
\item KRnet\_aug\_R\&N: The rotation layers and the nonlinear invertible layer are switched on for KRnet\_aug, where the rotation only acts on the data dimensions and does not affect the augmented dimension.
\item KRnet\_ODE: This is the neural ODE model based on the KRnet. 
\end{itemize}
For the numerical experiments, we obtain $6.4\times10^5$ samples from the mixture of Gaussians for the training set. We minimize the cross entropy between the empirical distribution and the model using 8 minibatches. The error is defined as the relative difference between the cross entropy and the differential entropy of the mixture of Gaussians, see equation \eqref{eqn:ce_error}, which can be regarded as the KL divergence between the model and the data distribution since the sample size is relatively large. For the KRnet\_ODE, the ODE is discretized on $[0,1]$ with a step size 0.05. 

All models have been trained using the same training set. For each model, we implement the training process ten times and define the mean of the ten errors as the final error. This way the bias from random initialization is reduced. For each training process, we run up to 8000 epochs. The results have been summarized in table \ref{tbl:compare_KRnets}. First of all, for each model the error decays as the number of affine coupling layers increases. Second, the KRnet\_ODE demonstrates a better performance than KRnet when $L$ is small, and is outperformed by KRnet when $L$ is large. However, KRnet\_ODE is significantly slower than KRnet. Third, the model KRnet\_aug\_R\&N yields the best performance, implying that the dimensional augmentation, the rotation layer and the nonlinear invertible layer are effective. When $L=2$, KRnet\_aug performs the worst. This is reasonable since the number of dimensions is increased by one. However, the KRnet\_aug has a fast decay in error. 

In figure \ref{fig:2d_mix_G_samples}, we compare the data distributions given by the training set, KRnet\_aug, and 
KRnet\_ODE for the case $L=6$ in Table \ref{tbl:compare_KRnets}. Both KRnet\_aug and KRnet\_ODE produce a  distribution that is visually the same as the data distribution given by the training set. 

\begin{figure}	
	\center{\includegraphics[width=0.7\textwidth]{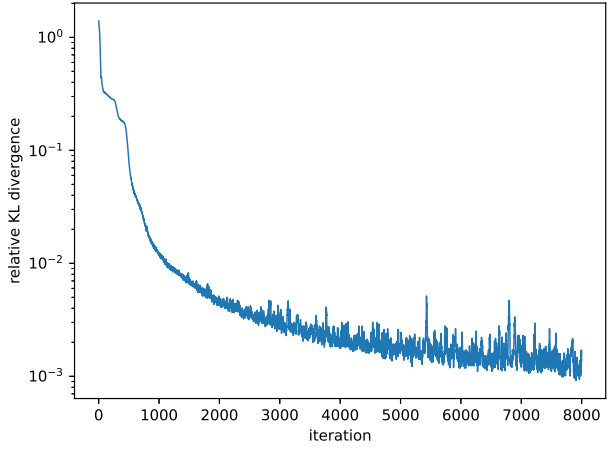}}
	\caption{The convergence behavior of KRnet\_aug\_R\&N for the approximation of the 2d mixture of Gaussians \eqref{eqn:mixture_of_G}.}\label{fig:2d_mix_G_approximation}
\end{figure}
\begin{figure}	
	\center{\includegraphics[width=0.99\textwidth]{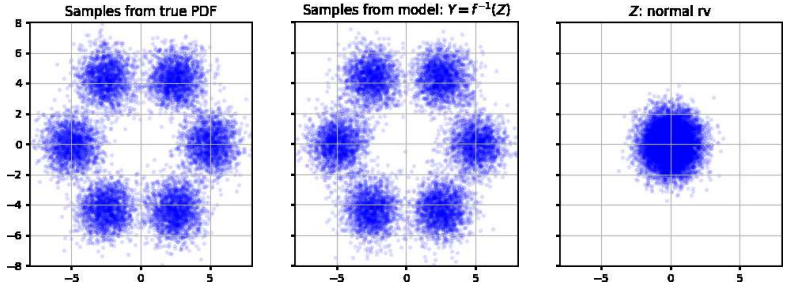}}
	\caption{Compare the data distributions from the 2d mixture of Gaussians \eqref{eqn:mixture_of_G} and the approximated PDF given by KRnet\_aug\_R\&N. The sample size is $N=10000$.}\label{fig:2d_mix_G_approx_samples}
\end{figure}

\begin{figure}	
	\center{\includegraphics[width=0.49\textwidth]{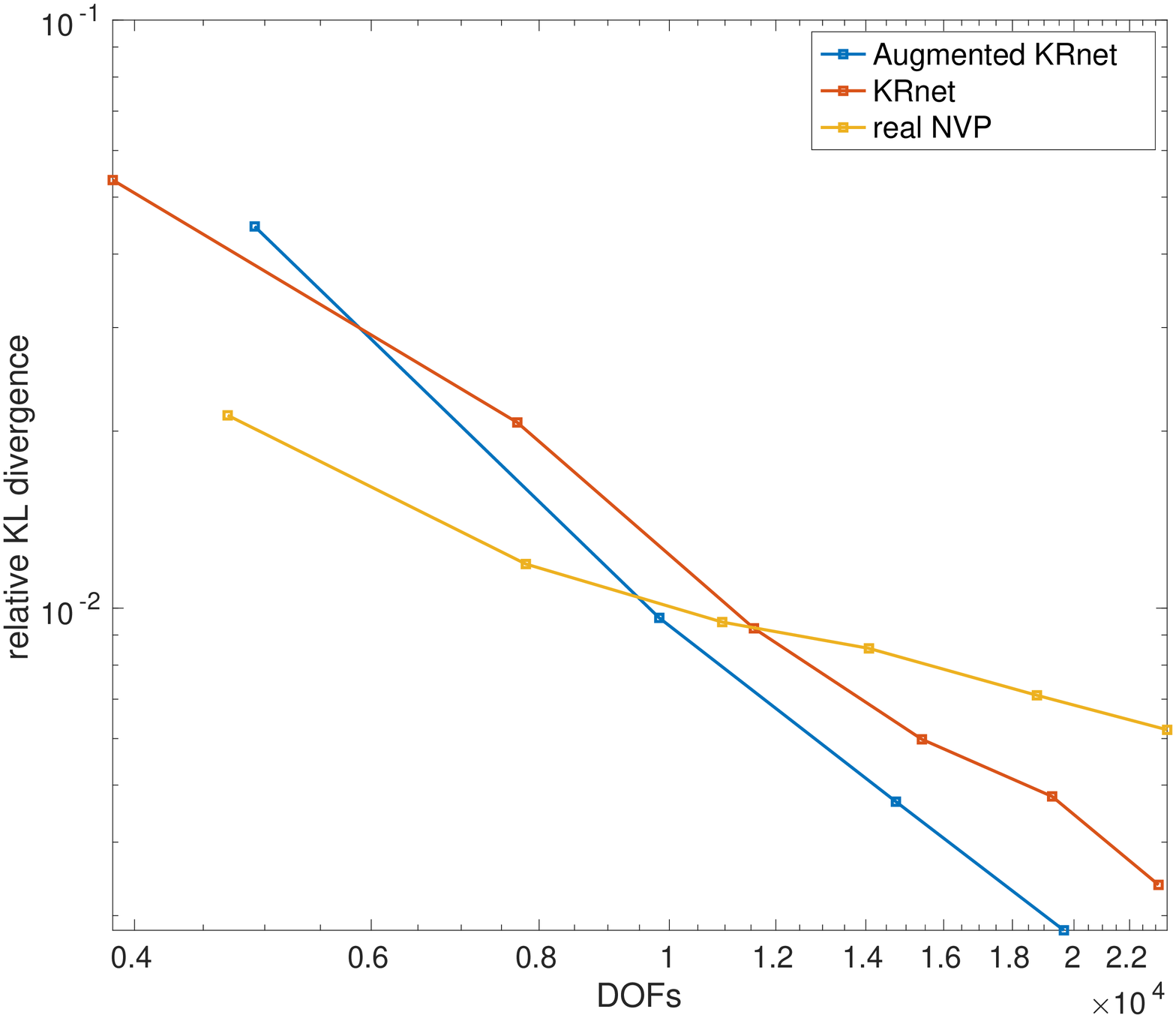}
		\includegraphics[width=0.49\textwidth]{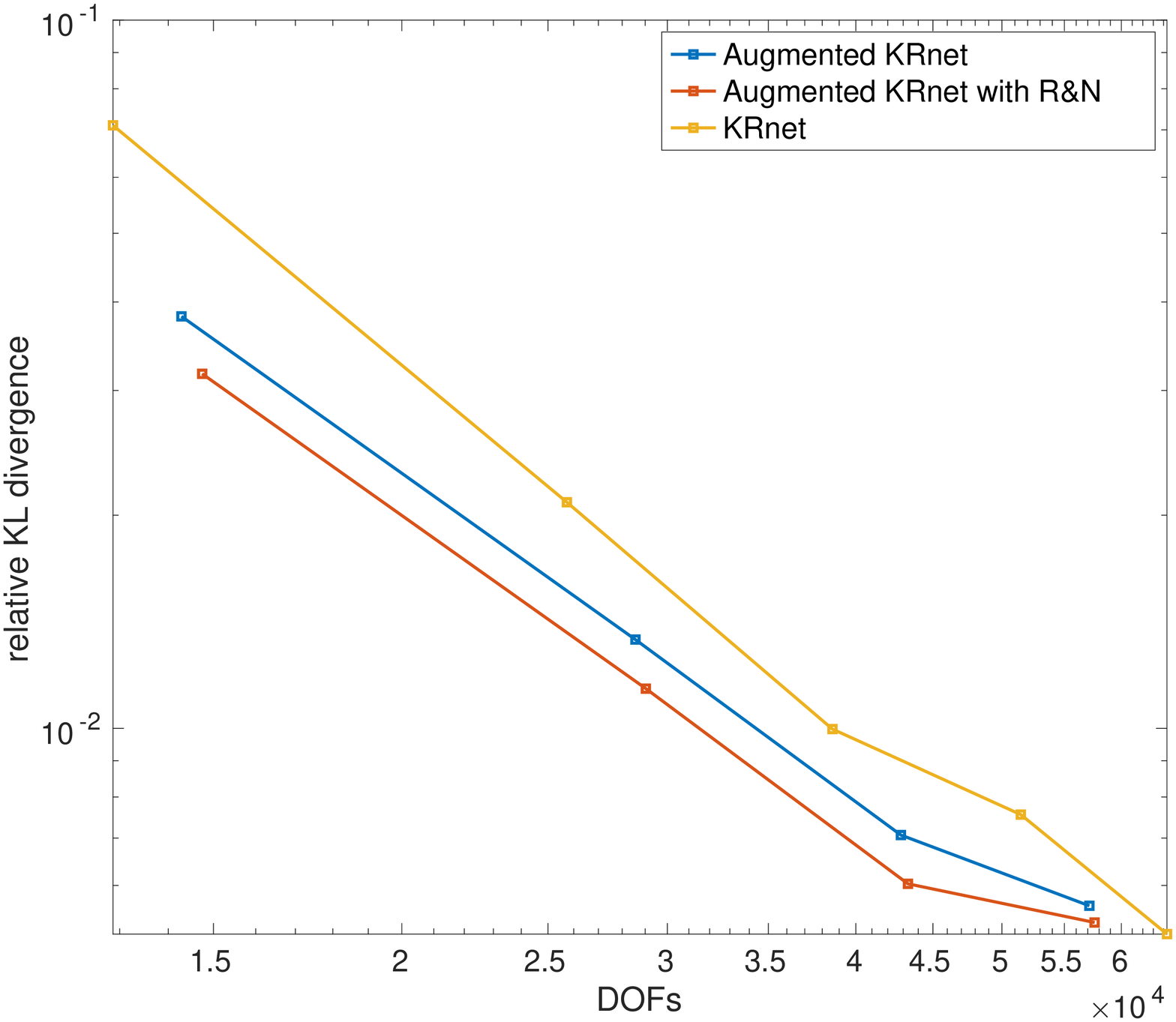}}
	\caption{Compare the convergence behavior of augmented KRnet, regular KRnet and real NVP. Left: $n=4$; Right: $n=8$.}\label{fig:D4_err_aKRnet}
\end{figure}
\begin{figure}	
	\center{\includegraphics[width=0.99\textwidth]{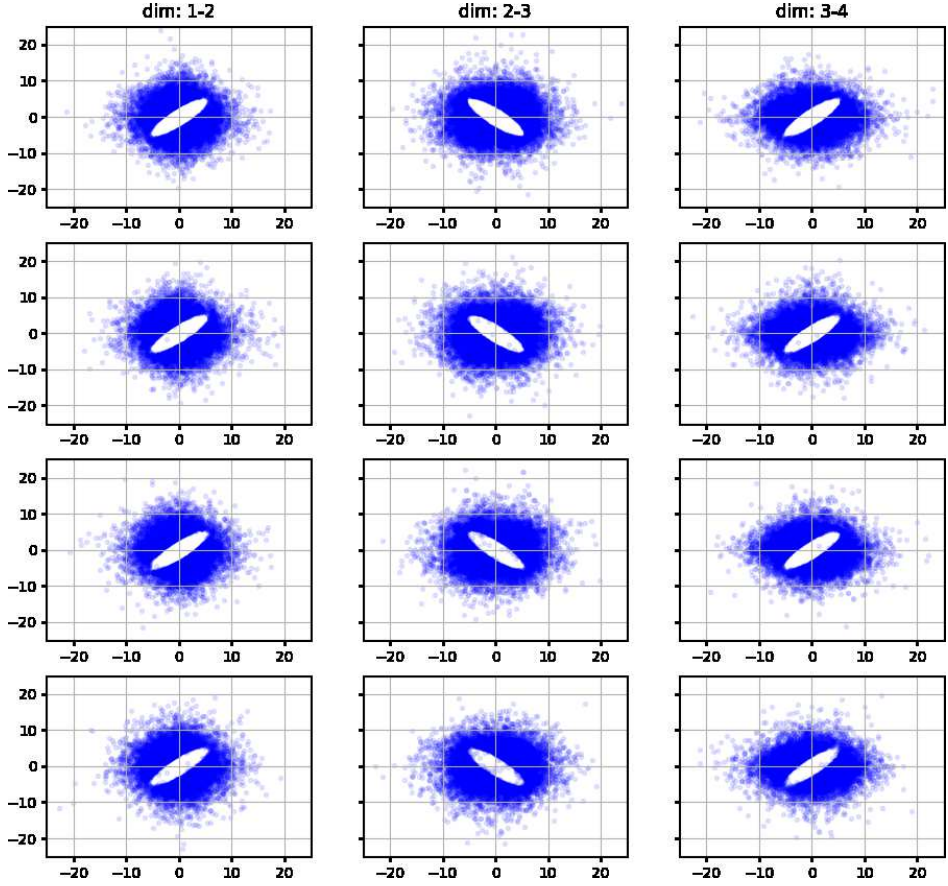}}
	\caption{Compare the samples from training set, the augmented KRnet, the regular KRnet and the real NVP for about the same number of DOFs. The three models correspond to the three cases in the left plot of figure \ref{fig:D4_err_aKRnet} with DOFs about 2e4. From top to bottom, each row shows three groups of adjacent dimensions ($(y_1,y_2)$, $(y_2,y_3)$, $(y_3,y_4)$) for the data from the training set, the augmented KRnet, the regular KRnet and the real NVP, respectively. Each set has 10000 samples.}\label{fig:4d_logistic_aKRnet_rNVP}
\end{figure}
\begin{figure}	
	\center{\includegraphics[width=0.99\textwidth]{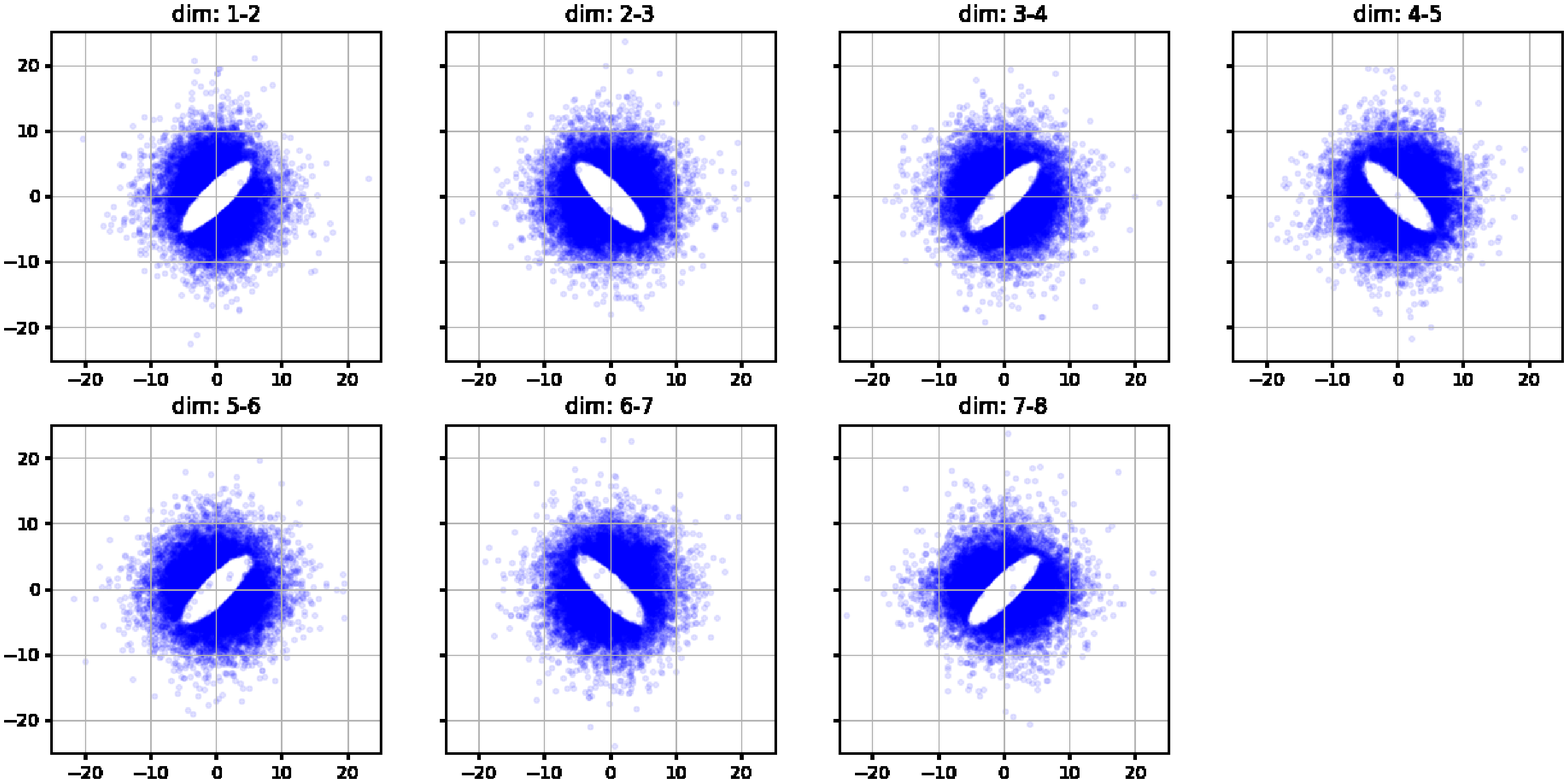}}
	\caption{Samples generated by augmented KRnet for the 8-dimensional Logistic distributions with elliptic holes. The sample size is $N=10000$. The model corresponds to the third case of KRnet\_aug\_R\&N in the right plot of figure \ref{fig:D4_err_aKRnet}, where the number of DOFs is about 4.3e4.}\label{fig:8d_logistic_samples_c1}
\end{figure}

We next consider the density approximation. We use KRnet\_aug\_R\&N to approximate the PDF \eqref{eqn:mixture_of_G} by minimizing the KL divergence \eqref{eqn:KL_for_approx}. For this case, there does not exist a training set. The samples for the approximation of the KL divergence are from the model KRnet\_aug\_R\&N. Since every minibatch can be independently sampled from the model, the optimization solver can be regarded as a minibatch stochastic gradient method with a training set of infinitely many data. For KRnet\_aug\_R\&N, we use $L=6$, and the rest of the configuration is the same as before. The size of minibatch is $10^5$. In figure \ref{fig:2d_mix_G_approximation}, we plot the convergence behavior of KRnet\_aug\_R\&N, and in figure \ref{fig:2d_mix_G_approx_samples}, we compare the samples from the true PDF and the approximated PDF. It is seen that the augmented KRnet is also effective for density approximation.

\subsection{Logistic distribution with holes}
The training data sets $\mathcal{S} = \{\by^{(i)}\}_{i=1}^{N_t}$ for density estimation are generated as follows. Assume that $\bY$ has i.i.d. components and each component $Y_i\sim\mathrm{Logistic}(0,s)$ with PDF $\rho(y_i;0,s)$. We propose the following constraint
\begin{equation}\label{eqn_logi_hole}
    \|{R}_{\gamma, \theta_j} [y_j^{(i)}, y_{j+1}^{(i)}]^\mathsf{T}\|_2 \geq C, \quad j = 1,\ldots,d-1,
\end{equation}
where $C$ is a specified constant, and 
\begin{equation*}
R_{\gamma, \theta_j} = \left[ \begin{array}{cc}
         \gamma & 0  \\
         0 & 1
    \end{array}
    \right] \left[ \begin{array}{cc}
         \mathrm{cos} \theta_j & -\mathrm{sin} \theta_j  \\
         \mathrm{sin} \theta_j & \mathrm{cos} \theta_j
    \end{array}
    \right], \quad \theta_j =  \frac{\pi}{4}, \ \text{if} \ j \ \text{is even}; 
    \frac{3\pi}{4}, \ \text{otherwise}.
\end{equation*}
We then generate samples $\by^{(i)}$ of $\bY$, out of which we only  accept those that satisfy the constraint \eqref{eqn_logi_hole}. This way, an elliptic hole is generated for two adjacent dimensions. The reference PDF takes the form
\begin{equation}
p_{\bY,\mathsf{ref}}(\by)=\frac{I_B(\by)\prod_{i=1}^n\rho(y_i;0,s)}{\mathbb{E}[I_{B}(\bY)]},
\end{equation}
where $B$ is the set defined by equation \eqref{eqn_logi_hole} and $I_B(\cdot)$ is an indicator function with $I_B(\by)=1$ if $\by\in B$; 0, otherwise. 

For this test problem, we set $\gamma =3$ and $C = 7.6$.  This case was studied in \cite{Wan_KRnet} and we use the same setup here. The size of the training set is $N_t=9.6\times10^5$ and the errors are computed in terms of a validation set of size $3.2\times10^5$. For each model configuration, we train the model 10 times respectively in terms of 10 independently sampled training sets. We then use the averaged error to reduce the bias. The neural network for the affine coupling layers has two dense hidden layers of 24 neurons for $n=4$ and of 32 neurons for $n=8$. The comparison of the augmented KRnet, the regular KRnet and the real NVP is summarized in figure \ref{fig:D4_err_aKRnet} in terms of DOFs, where the relative Kullback-Leibler (KL) divergence is defined as 
\[
\frac{D_{\mathsf{KL}}(p_{\bY,\mathsf{ref}}\|p_{\bY})}{h(p_{\bY,\mathsf{ref}})},
\]
where $D_{\mathsf{KL}}(p_{\bY,\mathsf{ref}}\|p_{\bY})$ is approximated by the validation set. It is seen that both the augmented KRnet and the regular KRnet yield a much better trend in terms of the convergence rate than the real NVP. The augmented KRnet and the regular KRnet have similar convergence behavior while the augmented KRnet is more effective than the regular KRnet for the same number of DOFs, which is verified by the simulation results for both $d=4$ and $d=8$. On the right plot, we also include the results for the augmented KRnet with rotation and nonlinear layers. With a slightly larger number of DOFs, the rotation and nonlinear layers further improve the performance of the augmented KRnet. It is seen that rotation and nonlinear layers do not improve the augmented KRnet for the last case. It is because a constant error has been reached since both the loss and the generalization error have been approximated by the Monte Carlo method.

In figure \ref{fig:4d_logistic_aKRnet_rNVP}, we have compared the samples generated by some generative models to the training set. The data distribution to be learned is highly irregular. On any face given by two adjacent dimensions, a cylinder hole exists. On the boundary of this cylinder, the density can be large. The existence of sharp discontinuities in density implies that classical PDF models such as the mixture of Gaussians are not effective. However, the deep generative models can deal with this high-dimensional density estimation problem quite well. Roughly speaking, we may tell the improvement from the real NVP to the augmented KRnet by the decreasing number of outliers in the hole, where the density is supposed to be zero. Out of 10000 samples only a few show up in the holes meaning that boundaries of the holes have been well captured. 

In figure \ref{fig:8d_logistic_samples_c1} we plot the distribution of samples generated by the augmented KRnet for the 8-dimensional Logistic distribution with elliptic holes. It is seen that for such a high-dimensional irregular distribution the sharp discontinuities in density can also be well resolved.

\section{Summary and discussions}
In this work we have developed augmented KRnet for both discrete and continuous models. The main idea is to introduce augmented dimensions to enhance the exchange of information between data dimensions such that the flow-based generative model induced by KRnet may further increase its modeling capability while maintaining the exact invertibility of the transport map. We have also formulated the augmented KRnet as a discretization of a neural ODE by a one-step method of first-order accuracy, where the exact invertibility has been kept locally. Although we are not able to discretize the neural ODE with a high-order numerical scheme, a dynamical model with a first-order invertible discretization is still of particular interest for the modeling of dynamical data since the gradient can be exactly computed. A number of numerical experiments have been implemented. Both discrete and continuous models based on the augmented KRnet are effective for both density estimation and approximation, where the algebraic convergence is observed as the number of DOFs increases. In particular, the augmented KRnets are able to deal with high-dimensional distributions that have sharp discontinuous boundaries. Based on these observations, we think that the augmented KRnet may serve as a generic PDF model for many applications. At this moment, our numerical experiments show that the discrete models are in general more effective and much faster than the continuous models. Further research is needed to improve the efficiency of  KRnet\_ODE.

\section*{Acknowledgment}
This work was supported by NSF grant DMS-1913163.


\begin{thebibliography}{100}
\bibitem{Arjovsky_2017}
M.~Arjovsky, S.~Chintala, and L.~Bottou,
\emph{Wasserstein GAN}, (2017), arXiv:1701.07875v3.

\bibitem{Blei_2018}
D.~M.~Blei, A.~Kucukelbir, and J.~D.~McAuliffe, 
\emph{Variational inference: A review for statisticians}, (2018), 
arXiv:1601.00670v9.

\bibitem{Berg_2019}
R.~van~den~Berg, L.~Hasenclever, J.~M.~Tomczak and M.~Welling, 
\emph{Sylvester normalizing flows for variational inference}, (2019), 
arXiv:1601.00670v9.

\bibitem{Carlier_2010}
G.~Carlier, A.~Galichon, and F.~Santambrogio,  
\emph{From Knothe’s transport to Brenier’s map and a continuation
method for optimal transport}, SIAM J.  Math. Anal., 41(6) (2010), pp. 2554--2576.

\bibitem{Chen_2019}
R.~T.~Q.~Chen, Y.~Rubanova, J.~Bettencourt, and D.~Duvenaud,   
\emph{Neural ordinary differential equations}, (2019), arXiv:1806.07366v5.

\bibitem{Dinh_2014}
L.~Dinh, D.~Krueger, and S.~Bengio, 
\emph{Nice: non-linear independent components estimation}, (2014), arXiv:1410.8516.
	
\bibitem{Dinh_2016}
L.~Dinh, J.~Sohl-Dickstein, and S.~Bengio, 
\emph{Density estimation using real NVP}, (2017), arXiv:1605.08803v3.
	
\bibitem{Dupont_2019}
E.~Dupont, A.~Doucet, and Y.~W.~Teh, 
\emph{Augmented neural ODEs}, (2019), arXiv:1904.01681v3.	

\bibitem{Filippo2010}
F.~Santambrogio, 
\emph{Optimal Transport for Applied Mathematicians}, Birkh\"{a}user, 2010.

\bibitem{Finlay_2020}
C.~Finlay, J.-H.~Jacobsen, L.~Nurbekyan, and A.~M.~Oberman 
\emph{How to train your neural ODE: the world of Jacobian and kinetic regularization}, (2020), arXiv:2002.02798v3.


\bibitem{Goodfellow_2014}
I.~Goodfellow, J.~Pouget-Abadie, M.~Mirza, B.~Xu, D.~Warde-Farley, S.~Ozair, A.~Courville, and Y.~Bengio, 
\emph{Generative adversarial nets}, Advances in Neural Information Processing Systems, (2014), 2672--2680.
	
\bibitem{Graves_2013}
A.~Graves, 
\emph{Generating sequences with recurrent neural networks}, (2013), arXiv:1308.0850.
	
\bibitem{Grover_2018}
A.~Grover, M.~Dhar, and S.~Ermon, 
\emph{Flow-GAN: Combining maximum likelihood and adversarial learning in generative models}, (2018), 
arXiv:1705.08868v2.	
	
\bibitem{He_2015}
K.~He, X.~Zhang, S.~Ren, and J.~Sun,  
\emph{Deep residual learning for image recognition}, (2015), 
arXiv:1512.03385v1.		
	
\bibitem{Szegedy_2015}
S.~Ioffe, and C.~Szegedy, 
\emph{Batch normalization: Accelerating deep network training by reducing internal covariance shift}, 
(2015), arXiv:1502.03167v3.	

\bibitem{Kingma_2014}
D.~P.~Kingma, and M.~Welling,
\emph{Auto-encoding variational Bayes}, (2014), 
arXiv:1312.6114v10.
	
\bibitem{Dhariwal_2018}
D.~P.~Kingma, and P.~Dhariwal, 
\emph{Glow: Generative flow with invertable 1x1 convolutions}, (2018), arXiv:1807.03039v2.
	
\bibitem{ADAM_2017}
D.~P.~Kingma, and J.~L.~Ba, 
\emph{ADAM: A method for stochastic optimization},(2017), arXiv:1412.6980v9.
	
\bibitem{Kingma_2016}
D.~P.~Kingma, T.~Salimans, R.~Jozefowicz, X.~Chen, I.~Sutskever, and M.~Welling, 
\emph{Improving variational inference with inverse autoregressive flow},  Advances in Neural Information Processing Systems, (2016), pp. 4743--4751.	
	
\bibitem{Lu_2020}
Y.~Lu, A.~Zhang, Q.~Li, and B.~Dong, 
\emph{Beyond finite layer neural networks: bridging deep architectures and numerical differential equations},(2020), arXiv:1710.10121v3.	
	
\bibitem{Oord_2016a}
A.~van~den~Oord, N.~Kalchbrenner, and K.~Kavukcuoglu,
\emph{Pixel recurrent neural networks}, (2016), arXiv:1601.06759.
	
\bibitem{Oord_2016b}
A.~van~den~Oord, N.~Kalchbrenner, O.~Vinyals, L.~Espeholt, A.~Graves, and K.~Kavukcuoglu,
\emph{Conditional image generation with PixcelCNN decoders}, (2016), arXiv:1606.05328.
	
\bibitem{Papamakarios_2018}
G.~Papamakarios, T.~Pavlakou, and I.~Murray, 
\emph{Masked autoregressive flow for density estimation}, (2018), arXiv:1705.07057v4.

\bibitem{Rezende_2015}
D.~Rezende, and S.~Mohamed, 
\emph{Variational inference with normalizing flows}, ICML, (2015), 1530--1538.


\bibitem{Scott2015}
D.~Scott, 
Multivariate Density Estimation: Theory, Practice, and Visualization, 2nd Edition, John Wiley \& Sons, Inc., 2015.

\bibitem{Spantini_2017}
A.~Spatini, D.~Bigoni, and Y.~Marzouk, 
\emph{Inference via low-dimensional couplings}, (2017), arXiv:1703.06131v4.
	
\bibitem{Wan_KRnet}
K~Tang, X.~Wan, and Q.~Liao, 
\emph{Deep density estimation via invertible block-triangular mapping}, Theoretical \& Applied Mechanics Letters, \textbf{10}, 2020, 000-5.
	
\bibitem{Wan_KRnet_FPE}
K~Tang, X.~Wan, and Q.~Liao, 
\emph{Adaptive deep density approximation for Fokker-Planck equations}, (2021), arXiv:2013.11181v1..	
	
\bibitem{Wan_JCP20}
X.~Wan, and S.~Wei,
\emph{Coupling the reduced-order model and the generative model for an importance sampling estimator}, J. Compt. Phys., in press.
		
\bibitem{Yang_2019}
L.~Yang, and G.~E.~Karniadakis, 
\emph{Potential flow generator with $L_2$ optimal transport regularity for generative models}, (2019), arXiv:1908.11462v1.	
		
\bibitem{Zhang_2018}
L.~Zhang, W.~E, and L.~Wang, 
\emph{Monge-\mbox{A}mp\'{e}re flow for generative modeling}, (2018), arXiv:1809.10188v1.
	

\bibitem{Zhu_2019}
J.~Zhu, D.~Zhao, and B.~Zhang, 
{LIA: Latently Invertible Autoencoder with Adversarial Learning}, (2019), arXiv:1906.08090v1.

\end{thebibliography}
\end{document}